\documentclass[mnsc,sglanonrev]{informs4}

\OneAndAHalfSpacedXI

\usepackage{mathtools}
\usepackage{algorithm}
\usepackage{algpseudocode}
\usepackage{booktabs}
\usepackage{enumitem}
\usepackage{subcaption}
\usepackage{float}
\usepackage{tikz}
\graphicspath{{../figs/}{figs/}{manuscript/figs/}}
\defcitealias{iab2026internet}{IAB/PwC 2026}
\usepackage[capitalize,noabbrev]{cleveref}
\makeatletter
\providecommand{\theHALG@line}{}
\renewcommand{\theHALG@line}{\thealgorithm.\arabic{ALG@line}}
\makeatother

\EquationsNumberedThrough

\TheoremsNumberedBySection
\ECRepeatTheorems

\makeatletter
\gdef\th@TH{%
  \def\theorem@headerfont{\normalfont\TheoremHeaderFont}%
  \theorempreskipamount=6pt plus 3pt minus 2pt
  \theorempostskipamount=6pt plus 3pt minus 2pt
  \if@MOOR\labelsep1em\else\labelsep0.5em\fi
  \def\@begintheorem##1##2{\normalfont\TheoremTextFont
        \item[\hskip\labelsep \theorem@headerfont ##1\ ##2.]}%
  \def\@opargbegintheorem##1##2##3{\normalfont\TheoremTextFont
        \item[\hskip\labelsep \theorem@headerfont ##1\ ##2\ {\bfseries(##3)}.]}}
\gdef\th@THkey{%
  \def\theorem@headerfont{\normalfont\TheoremHeaderFont}%
  \theorempreskipamount=6pt plus 3pt minus 2pt
  \theorempostskipamount=6pt plus 3pt minus 2pt
  \if@MOOR\labelsep1em\else\labelsep0.5em\fi
  \def\@begintheorem##1##2{\normalfont\TheoremTextFont
        \item[\hskip\labelsep \theorem@headerfont ##1\ ##2.]}%
  \def\@opargbegintheorem##1##2##3{\normalfont\TheoremTextFont
        \item[\hskip\labelsep \theorem@headerfont {##3}]}}
\makeatother

\theoremstyle{TH}
\newtheorem{abstraction}{Abstraction}[section]
\theoremstyle{EX}
\newtheorem{fact}{Fact}[section]

\renewenvironment{proof}[1][Proof]{%
  \par\addvspace{6pt plus 3pt minus 2pt}%
  \noindent\textit{#1.}\quad\ignorespaces
}{%
  \unskip\Halmos\par\addvspace{6pt plus 3pt minus 2pt}%
}

\makeatletter
\renewenvironment{assumption}[1][]{%
  \par\addvspace{6pt plus 3pt minus 2pt}%
  \noindent\refstepcounter{assumption}%
  {\normalfont\TheoremHeaderFont Assumption \theassumption
    \def\@tempa{#1}%
    \ifx\@tempa\@empty\else\ {\bfseries(#1)}\fi.}%
  \enskip\normalfont\TheoremTextFont\ignorespaces
}{%
  \par\addvspace{6pt plus 3pt minus 2pt}%
}

\renewenvironment{remark}[1][]{%
  \par\addvspace{6pt plus 3pt minus 2pt}%
  \noindent\refstepcounter{remark}%
  {\normalfont\TheoremHeaderFont Remark \theremark
    \def\@tempa{#1}%
    \ifx\@tempa\@empty\else\ {\bfseries(#1)}\fi.}%
  \enskip\normalfont\TheoremTextFont\ignorespaces
}{%
  \par\addvspace{6pt plus 3pt minus 2pt}%
}
\makeatother

\providecommand{\argmax}{\operatorname*{arg\,max}}
\providecommand{\argmin}{\operatorname*{arg\,min}}
\newcommand{\E}{\mathbb{E}}
\newcommand{\R}{\mathbb{R}}
\newcommand{\bc}{{\overline{c}}}
\newcommand{\1}{\mathbf{1}}
\newcommand{\calP}{\mathcal{P}}

\newcommand{\calC}{\mathcal{C}}
\newcommand{\hcalC}{\widehat{\mathcal{C}}}
\newcommand{\Clow}{\widehat{\mathcal{C}}^{\mathrm{low}}}

\newcommand{\calD}{\mathcal{D}}
\newcommand{\calI}{\mathcal{I}}

\newcommand{\calG}{\mathcal{G}}
\newcommand{\calB}{\mathcal{B}}

\newcommand{\hF}{{\widehat{F}}}
\newcommand{\hG}{{\widehat{G}}}
\newcommand{\hv}{{\widehat{v}}}
\newcommand{\hq}{{\widehat{q}}}
\newcommand{\hc}{{\widehat{c}}}
\newcommand{\hl}{{\widehat{\ell}}}
\newcommand{\htheta}{{\widehat{\theta}}}
\newcommand{\hphi}{{\widehat{\phi}}}
\newcommand{\te}{{\widetilde{e}}}
\newcommand{\bg}{{\overline{g}}}
\newcommand{\OPT}{\mathrm{OPT}}
\newcommand{\Var}{\mathrm{Var}}
\newcommand{\Reg}{\mathrm{Regret}}
\newcommand{\Viol}{\mathrm{Viol}}
\newcommand{\hob}{\mathrm{HoB}}
\newcommand{\bgt}{{\mathrm{bgt}}}

\newcommand{\full}{{\mathrm{full}}}
\newcommand{\bin}{{\mathrm{bin}}}
\newcommand{\ros}{{\mathrm{ros}}}

\DeclarePairedDelimiter{\abs}{\lvert}{\rvert}

\DeclarePairedDelimiter{\parr}{(}{)}
\DeclarePairedDelimiter{\parq}{[}{]}

\DeclarePairedDelimiter{\bra}{\lbrace}{\rbrace}

\RUNAUTHOR{Hu, Wen, Yao, Zhang, and Zhou}
\RUNTITLE{First-Price Auctions with Unknown Values and Constraints}
\TITLE{Learning to Bid with Unknown Private Values in Budget-Constrained First-Price Auctions}
\ARTICLEAUTHORS{%
\AUTHOR{Zihao Hu\textsuperscript{1,3,*}, Yuxiao Wen\textsuperscript{2,*},
Yuan Yao\textsuperscript{1}, Jiheng Zhang\textsuperscript{1,3}, and
Zhengyuan Zhou\textsuperscript{4}}
\AFF{\textsuperscript{1}Department of Mathematics, The Hong Kong University
of Science and Technology\\
\textsuperscript{2}Courant Institute of Mathematical Sciences, New York University\\
\textsuperscript{3}Department of Industrial Engineering and Decision Analytics,
The Hong Kong University of Science and Technology\\
\textsuperscript{4}Stern School of Business, New York University\\
\textsuperscript{*}Zihao Hu and Yuxiao Wen contributed equally and are listed
in alphabetical order.\\
\EMAIL{\{zihaohu, yuany, jiheng\}@ust.hk, yuxiaowen@nyu.edu,
zz26@stern.nyu.edu}}
}
\ABSTRACT{
\DIFadd{We study the operational problem of automated bidding in repeated
first-price auctions under budget and return-on-spend (RoS) constraints.} In this setting, an auto-bidder must translate advertiser
goals and constraints into real-time bids while learning two latent objects:
the causal uplift value of each ad impression and the highest competing bid (HoB)
needed to win it. We model uplift values and HoBs through a shared-context
Linear Treatment Effect (LTE) structure and analyze both full-information and
binary HoB feedback. 
We develop Dual-LTE, a dual-aware online learning framework that coordinates
value estimation, HoB estimation, and budget/RoS control through
confidence-guided exploration. We prove regret and constraint-violation
guarantees that scale as
\(\widetilde O(\sqrt T)\) under full-information HoB feedback and
\(\widetilde O(T^{2/3})\) under binary win/loss feedback, where
\(\widetilde O(\cdot)\) hides problem-dependent and logarithmic factors.
\DIFadd{Semi-synthetic experiments using real auction covariates show that
Dual-LTE achieves lower regret than the baselines across budget and RoS
settings, while illustrating the tradeoff between regret and constraint
violation.}
Our results provide operational guidance for DSPs and platform auto-bidders
that \DIFadd{manage advertiser budgets or seek to meet ROAS targets}. When impression values must be
learned, value estimation should be coordinated with budget or ROAS control:
\DIFadd{the auto-bidder should follow the Lagrangian bidding rule only when
value estimates are sufficiently accurate given the current constraint pressure
and should otherwise use controlled exploration.}}
\RequirePackage[normalem]{ulem} 
\RequirePackage{color}\definecolor{RED}{rgb}{1,0,0}\definecolor{BLUE}{rgb}{0,0,1} 
\newif\ifshowchanges
\showchangesfalse
\ifshowchanges
    \providecommand{\DIFadd}[1]{{\protect\color{blue}\uwave{#1}}} 
    \providecommand{\DIFdel}[1]{{\protect\color{red}\sout{#1}}} 
\else
    \providecommand{\DIFadd}[1]{#1}
    \providecommand{\DIFdel}[1]{}
\fi
\providecommand{\DIFaddbegin}{} 
\providecommand{\DIFaddend}{} 
\providecommand{\DIFdelFL}[1]{} 
\RequirePackage{listings} 
\RequirePackage{color} 
\lstdefinelanguage{DIFcode}{ 
  moredelim=[il][\color{red}\sout]{\%DIF\ <\ }, 
  moredelim=[il][\color{blue}\uwave]{\%DIF\ >\ } 
} 
\lstdefinestyle{DIFverbatimstyle}{ 
	language=DIFcode, 
	basicstyle=\ttfamily, 
	columns=fullflexible, 
	keepspaces=true 
} 
\lstnewenvironment{DIFverbatim}[1][]{\lstset{style=DIFverbatimstyle}}{} 
\lstnewenvironment{DIFverbatim*}[1][]{\lstset{style=DIFverbatimstyle,showspaces=true}}{} 
\begin{document}
\maketitle

\section{Introduction}


\DIFadd{Programmatic advertising's shift from second-price to first-price
auctions (FPAs)} \citep{bigler2019rolling,wong2021moving} \DIFadd{has
fundamentally changed how advertisers purchase impressions. Under the standard
private-value benchmark, truthful bidding is optimal in a second-price auction,
whereas a first-price bidder must determine how far to shade its bid below
value. In practice, advertisers delegate auction-level bid decisions to
demand-side platforms (DSPs) and other auto-bidding platforms, whose bidding
agents translate high-level goals and constraints into real-time bids}
\citep{aggarwal2024auto}. \DIFadd{Prominent examples include Google Display \&
Video 360, Amazon DSP, and The Trade Desk. These decisions are made at
substantial scale: IAB/PwC reports that U.S. non-search programmatic advertising
revenues reached \$162.4 billion in 2025} \citepalias{iab2026internet}.

\DIFadd{This delegation creates a challenging operational problem: the
auto-bidder must determine which impressions are valuable, how far to shade
bids below value, and how aggressively to spend over time while satisfying the
advertiser's budget or return-on-ad-spend (ROAS) constraint}. \DIFadd{These constraints link bidding decisions
across time. Aggressive bidding can deplete the budget early or reduce realized
returns, whereas conservative bidding can forgo valuable impressions and delay
learning. Effective auto-bidding must therefore balance immediate value
capture, constraint satisfaction, and the information gained for future
bidding decisions.}

\DIFadd{A central difficulty in designing such auto-bidding policies is that
the value of an impression is not directly observed before bidding. Platforms
such as Google Ads describe value-based Smart Bidding as using auction-time
contextual signals and advertiser-reported conversion data to predict
conversion value and optimize bids toward budget or ROAS goals}
\citep{google2026smartbidding,google2026targetroas}. \DIFadd{Such predictions
are operationally useful, but predicted conversion value is not the same as
the bid-relevant causal value. The latter is the incremental value created by
winning the auction and showing the ad relative to the counterfactual outcome
without the ad. A user may still purchase through organic search, direct
visit, or another channel if the advertiser loses the auction. Consequently,
a conversion observed after an ad is shown may reflect both the ad's causal
effect and the user's pre-existing propensity to purchase. The auto-bidder
must therefore learn the latent uplift value rather than equate predicted
conversion value with the causal effect of showing the ad}
\citep{lewis2015unfavorable,gordon2019comparison}.

\DIFadd{The impression characteristics available to the auto-bidder may also
be informative about auction competition. For example, an impression in a
prominent on-screen position may be attractive to many advertisers, leading to
stronger competition and a higher HoB. More generally, an impression that
appears valuable to one advertiser may also attract stronger demand from
others. The auto-bidder must therefore jointly learn both the causal value of
an impression and the price required to acquire it.}

This joint learning problem under constraints motivates our dual-aware
auto-bidding framework, which jointly learns uplift values and highest
competing bids under budget or return-on-spend (RoS) constraints. The central
managerial message is that value learning cannot be separated from constraint
management.
Exploratory bids improve future estimates but also consume budget or affect RoS
performance. \DIFadd{Mistaken exploitation is also costly because
value-estimation error may lead the auto-bidder to spend on low-value
impressions, thereby depleting the budget or worsening RoS performance.} Our framework therefore uses a confidence-guided
exploitation rule: it follows the learned bidding rule only when the value and
HoB estimates are accurate enough for the current budget or RoS pressure, and
otherwise uses the round for controlled exploration.

\noindent\textbf{Our Contributions.}
Our contributions are fourfold.

\begin{itemize}
    \item \textbf{Constrained auto-bidding with latent causal values.}
    We formulate repeated first-price auto-bidding under budget and RoS
    constraints when the advertiser's value for each ad impression is a latent
    causal uplift rather than an observed input. The model extends the jointly
    linear value--HoB setting of \citet{wen2025joint} from
    unconstrained bidding to \DIFadd{operational settings governed by budget
    and RoS constraints}.

    \item \textbf{Dual-aware learning with confidence-guided exploitation.}
    We develop a unified learning framework for budget- and RoS-constrained
    FPAs. \DIFadd{The algorithm follows the constraint-adjusted bid only when
    value-estimation uncertainty is sufficiently small relative to the current
    constraint pressure and otherwise conducts controlled exploration.}

    \item \textbf{Performance guarantees under full and binary HoB feedback.}
    We prove regret and constraint-violation guarantees under both
    full-information HoB feedback and binary win/loss feedback. The rates scale
    as \(\widetilde O(\sqrt T)\) under full-information HoB feedback and as
    \(\widetilde O(T^{2/3})\) under binary feedback, matching the known horizon
    rates for unconstrained LTE bidding.

    \item \textbf{Semi-synthetic evaluation on real auction covariates.}
    We evaluate the algorithms on semi-synthetic experiments based on the
    iPinYou real-time bidding dataset~\citep{zhang2014real}. \DIFadd{The
    experiments show that, under both full-information and binary feedback,
    Dual-LTE achieves lower regret than the baselines in the budget-constrained
    settings and favorable regret--violation tradeoffs under RoS constraints.}
\end{itemize}

\subsection{Related Work}

\noindent\textbf{Learning in Online Ad Auctions:}
A growing literature studies repeated bidding and learning in online ad
auctions. See \citet{aggarwal2024auto} for a recent survey. Theoretical work on
first-price auctions studies equilibrium behavior, pacing, and efficiency in
strategic auction markets~\citep{conitzer2022pacing,gaitonde2022budget,
lucier2024autobidders}. A complementary online-learning literature designs
bidding policies with known values, HoB learning, or contextual auction
feedback~\citep{han2025optimal,balseiro2019learning,
badanidiyuru2023learning,zhang2024online,wang2023learning}. The closest work on
HoB learning is \citet{badanidiyuru2023learning}, which models the opponents'
highest bid using a linear model but assumes that the bidder's value is observed
before bidding. We instead study a jointly linear setting in which both the
HoB process and the bidder's uplift value are unknown and depend linearly on
the context.

\noindent\textbf{Online Bidding with Constraints:}
Our problem is related to bandits with knapsacks and online optimization with
long-term constraints~\citep{agrawal2016linear,badanidiyuru2018bandits,
immorlica2022adversarial,castiglioni2022online,castiglioni2022unifying}. These
works provide general tools for balancing reward and constraint satisfaction,
\DIFadd{and can be instantiated for constrained bidding. However, they treat
rewards and resource consumption as generic action-dependent outcomes and
therefore do not exploit the specific structure of first-price auctions, in
which a bid governs both the probability of winning and, conditional on
winning, the payment}. In online
advertising, budget pacing and ROI/RoS constraints are central to auto-bidding
systems~\citep{balseiro2023field,lucier2024autobidders}. Within online
bidding, \citet{wang2023learning} study first-price auctions with budgets,
\citet{feng2023online} study second-price auctions under an RoS constraint,
\citet{li2025no} study repeated first-price auctions with ROI/RoS-style
constraints, \citet{aggarwal2025no} study non-truthful auctions with budget and
ROI constraints, and \citet{vijayan2025online} study RoS bidding in which value
is observed after a winning bid. These constrained-bidding models still rely
on impression values that are known before bidding or observed afterward. \DIFadd{Our setting combines
constrained first-price bidding, latent uplift-value learning, and HoB learning:
the auto-bidder must learn both the causal value of an impression and the price
required to acquire it while operating under budget or RoS constraints.}

\noindent\textbf{Causal Measurement in Advertising:}
Causal inference is increasingly used in operations and platform settings to
evaluate interventions, detect heterogeneous treatment effects, and guide
data-driven decisions~\citep{cui2022gender,wang2022instrumental,
wang2023personalized,ding2022buyers}. In online advertising, randomized
experiments and observational measurement estimate advertising
lift~\citep{kohavi2020trustworthy,gordon2019comparison}, while counterfactual
learning and online causal-inference methods use logged or auction data to
evaluate advertising effects~\citep{bottou2013counterfactual,
waisman2019online}. Other work uses randomized platform experiments to
evaluate advertising algorithms and combinations of operational interventions
\citep{ye2023cold,ye2025deep}. Recent causal bidding work embeds
value learning directly into auction decisions: \citet{waisman2019online}
identify causal advertising effects from optimal bids in real-time auctions
and develop a Thompson-sampling method that learns these bids while accounting
for the costs of experimentation, \DIFadd{and }\citet{wen2026marginal} estimate
marginal ad value in repeated second-price auctions. \DIFadd{Closest to our
work, }\citet{wen2025joint} \DIFadd{study joint
value and HoB learning in repeated first-price auctions without budget or RoS
constraints.} We instead
study constrained first-price auto-bidding in which the auto-bidder must
jointly learn causal value and the HoB while managing budget or RoS constraints. This coupling is nontrivial
because exploration and mistaken exploitation both consume budget or RoS slack,
so treatment-effect learning must be coordinated with constraint management.

\subsection{Roadmap}

Section~\ref{sec:formulation} introduces the repeated first-price auction model,
the budget and RoS benchmarks, and the standing assumptions.
Section~\ref{sec:hob_est} develops the HoB-estimation oracle and its
full-information and binary-feedback implementations.
Section~\ref{sec:dual_aware_lte_framework} presents the dual-aware LTE
framework, including value estimation and the unified bidding algorithm.
Section~\ref{sec:regret_instantiations} states the regret
and violation guarantees for the budget and RoS instantiations.
Section~\ref{sec:numerical_experiments} reports the semi-synthetic iPinYou
experiments, and Section~\ref{sec:conclusion} \DIFadd{concludes}.

\section{Problem Formulation}
\label{sec:formulation}

\smallskip
\noindent\textit{Notation.}\quad
For an integer \(n\), write \([n]=\{1,\ldots,n\}\). We use
\(\E[\cdot]\) for expectation and \(\1\{\cdot\}\) for indicators. The notations
\(O(\cdot)\) and \(\Theta(\cdot)\) have their standard asymptotic meanings. The
notation \(\widetilde O(\cdot)\) suppresses logarithmic factors in \(T\),
confidence parameters, and fixed problem constants. \DIFadd{When summarizing
horizon rates such as \(\widetilde O(\sqrt T)\), we report only the dependence
on \(T\).} The modes \(\full\) and \(\bin\) denote full-information
and binary HoB feedback, while \(\bgt\) and \(\ros\) denote the budget and
return-on-spend constraint regimes.

We consider a single bidder in repeated first-price auctions over rounds
\(t\in[T]\). \DIFadd{At the beginning of round \(t\), the bidder observes a
context vector \(x_t\in\R^d\) representing the available information about
the ad impression and auction, with \(\|x_t\|_2\le1\), and submits a bid
\(b_t\in[0,1]\).} \DIFadd{Motivated by practical bidding systems, we restrict
the bidder to a finite bid grid \(\calB\subseteq[0,1]\).} \DIFadd{Contexts
\(x_t\) are drawn i.i.d. over time, and the conditional distributions of the
potential outcomes and HoB given \(x_t\) are time-invariant.} The bidder's latent uplift value and the highest competing bid (HoB)
follow the jointly linear model
\begin{equation}\label{eq:induce}
    v_{t,1}-v_{t,0}=\theta_*^\top x_t+\epsilon_t,\qquad
    m_t=\phi_*^\top x_t+\xi_t .
\end{equation}
where \(\theta_*,\phi_*\in\R^d\) are unknown,
\(\E[\epsilon_t\mid x_t]=\E[\xi_t\mid x_t]=0\), \(m_t\in[0,1]\), and \(G\) is
the noise CDF of \(\xi_t\). This specification induces the conditional
win-probability function
\(F_t(b)=\Pr(m_t\le b\mid x_t)=G(b-\phi_*^\top x_t)\), where \(m_t\) is the
HoB threshold. The HoB residual \(\xi_t\) is conditionally
sub-Gaussian given \(x_t\). \DIFadd{We assume \(\|\theta_*\|_2\le1\) and
\(\|\phi_*\|_2\le1\).} We normalize bids, HoBs, and potential outcomes so
that \(b_t,m_t,v_{t,0},v_{t,1}\in[0,1]\). We also assume the incremental
value is nonnegative, \(v_{t,1}-v_{t,0}\in[0,1]\). Hence the
conditional treatment-effect value is
\(v_t=\E[v_{t,1}-v_{t,0}\mid x_t]=\theta_*^\top x_t\).

For any bid \(b\), let \(o_t(b)=\1[b\ge m_t]\) indicate whether the bidder wins
the auction, and let \(c_t(b)=b\cdot o_t(b)\) denote the first-price payment. If the bidder wins, it
observes \(v_{t,1}\) and pays \(b\). Otherwise it observes \(v_{t,0}\) and pays
zero. Under full-information
HoB feedback the learner observes \(m_t\), while under binary HoB
feedback it observes only \(o_t(b_t)\). To make the causal estimation of the marginal value $v_{t,1}-v_{t,0}$ tractable, we make the following assumptions.

\begin{assumption}[HoB-outcome conditional independence]\label{ass:unconfound}
The potential outcomes and the HoB are conditionally independent
given context, \((v_{t,1},v_{t,0})\perp\!\!\!\perp m_t\mid x_t\).
\end{assumption}

\begin{assumption}[History conditional independence]\label{ass:cond_ind}
\DIFadd{Let \(\mathcal H_{t-1}\) denote the information available to the
bidder after round \(t-1\), including its past actions, observations, and
randomization.} Conditional on
\(x_t\), the auction random vector \((v_{t,1},v_{t,0},m_t)\) is independent
of \(\mathcal H_{t-1}\).
\end{assumption}
Assumptions~\ref{ass:unconfound} and \ref{ass:cond_ind} follow the setup used
for uplift value estimation in FPAs~\citep{waisman2019online,wen2025joint}. The
first identifies the treatment effect from allocation-based outcomes: after
conditioning on the observed auction context, the HoB carries no
additional information about the potential outcomes. \DIFadd{The second
places the analysis in a stochastic, nonadaptive environment:
conditional on \(x_t\), the underlying auction variables are independent of past
history, although the bidder's policy may adapt to past observations.}
\DIFadd{The observed context should capture factors that jointly influence
user outcomes and auction competition; otherwise, uplift estimates may
confound the effect of advertising with selection into easier or more
competitive auctions. History-conditional independence is plausible in large
advertising markets, where many advertisers participate and the actions of any
single bidder have a negligible effect on future auction conditions over the
campaign horizon.}

\begin{assumption}[Noise CDF regularity and no zero-price allocation]\label{ass:bound_density}
The noise CDF \(G\) is \(L\)-Lipschitz, and \(F_t(0)=0\) for every round
\(t\).
\end{assumption}
Lipschitzness controls HoB CDF approximation errors, similar to
\citet{badanidiyuru2023learning}. The boundary condition means that a zero bid
cannot win, consistent with real-world auctions in which competition, reserve
prices, or bid floors impose a positive winning
threshold~\citep{OpenRTB2.6,google_authorized_buyers_bid_data}.

Following prior work on constrained auto-bidding, we model budget bidding as net
surplus maximization~\citep{wang2023learning} and RoS bidding as
gross value maximization subject to a return-on-spend
requirement~\citep{feng2023online,li2025no,vijayan2025online}. Thus the reward is defined respectively as
\[
    r_t^{\bgt}(b)=o_t(b)(v_{t,1}-v_{t,0}-b),
    \qquad
    r_t^{\ros}(b)=o_t(b)(v_{t,1}-v_{t,0}).
\]
\DIFadd{For a target \(\kappa>0\), the RoS constraint requires expected
cumulative value to be at least \(\kappa\) times expected cumulative spend:}
\[
    \DIFadd{\sum_{t=1}^T\E[r_t^{\ros}(b_t)]
    \ge
    \kappa\sum_{t=1}^T\E[c_t(b_t)].}
\]
\DIFadd{The corresponding expected one-round spend and RoS slack are
\(\bar c_t(b)=bF_t(b)\) and \(g_t(b)=(v_t-\kappa b)F_t(b)\), respectively. The cumulative expected RoS
violation is}
\[
    \Viol_T^{\ros}=\left[-\sum_{t=1}^T\E g_t(b_t)\right]_+ .
\]
In the benchmark definitions below, policies may be randomized and map
contexts to distributions over bids. Expressions such as
\(r_t^{\bgt}(\pi(x_t))\), \(r_t^{\ros}(\pi(x_t))\), and \(c_t(\pi(x_t))\)
include the policy randomization.

\smallskip
\noindent\textit{Budget benchmark and regret.}\quad
With a hard budget \(B\), the algorithm stops whenever the bid would exceed
the remaining budget. Let \(\tau\in\{0,\ldots,T\}\) be the last round
in which the algorithm submits a bid, with \(\tau=T\) if the budget is never
exhausted. Define the optimal per-round value of a stationary budget-feasible
policy by
\[
    \OPT^{\bgt}
    \coloneqq
    \max_{\pi}
    \E\!\left[r_t^{\bgt}(\pi(x_t))\right],
    \qquad
    \textnormal{s.t.}\quad
    \E\!\left[c_t(\pi(x_t))\right]\le \frac{B}{T}.
\]
Let \(\pi_{\bgt}^*\) be any policy attaining this maximum. By stationarity,
its expected cumulative reward is \(T\OPT^{\bgt}\). Accordingly, the
budget-constrained regret is
\[
    \Reg_T^{\bgt}
    =
    T\OPT^{\bgt}
    -
    \E\!\left[\sum_{t=1}^{\tau}r_t^{\bgt}(b_t)\right].
\]
We make the following linear budget assumption, which keeps the per-period
budget on a constant scale, consistent with standard assumptions in online
bidding with budgets~\citep{balseiro2019learning,wang2023learning}.
\begin{assumption}[Linear-growth budget]\label{ass:linear_budget}
The budget $B=cT$ for some constant $c\in(0,1]$.
\end{assumption}

\smallskip
\noindent\textit{RoS benchmark and convexification.}\quad
\DIFadd{The budget and RoS constraints both apply to aggregate performance, so
their benchmarks allow randomized bidding policies. In the RoS setting, the
benefit of randomization is especially transparent: a bid that generates
positive RoS slack can be mixed with a higher-value bid that consumes some of
that slack, thereby increasing expected value while maintaining aggregate RoS
feasibility}
\citep{li2025no}.

\DIFadd{For a fixed context \(x_t\), each deterministic bid \(b\) induces the
allocation--payment pair \((F_t(b),bF_t(b))\). Both expected value and RoS
slack depend on the bid only through this pair. A randomized bid distribution
induces a convex combination of the corresponding allocation--payment points.
We therefore define the per-context action set as}
\[
    \DIFadd{\calC_t
    =
    \operatorname{conv}\bigl\{(F_t(b),bF_t(b)):b\in[0,1]\bigr\}.}
\]
\DIFadd{Here \(\operatorname{conv}\) denotes the convex-hull operation. This
convexification is an exact representation of randomized bidding
rather than a relaxation: every point in \(\calC_t\) is induced by some
distribution over bids.

A randomized policy induces an action
\((q_\pi(x_t),c_\pi(x_t))\in\calC_t\), where \(q_\pi(x_t)\) is its allocation
probability and \(c_\pi(x_t)\) is its expected spend. Expected value and RoS
slack are then \(v_tq_\pi(x_t)\) and
\(v_tq_\pi(x_t)-\kappa c_\pi(x_t)\), respectively, and are linear in the
induced action. The finite-grid representation and implementation of these
actions are developed in Section~\ref{sec:lower_conv_hull}.}

The stationary convexified RoS benchmark is
\begin{equation}\label{eq:opt_def_contextual}
    \OPT^{\ros}
    =
    \sup_{\pi}
    \E_{x_t}[v_tq_\pi(x_t)]
    \quad
    \textnormal{s.t.}\quad
    \E_{x_t}[v_tq_\pi(x_t)-\kappa c_\pi(x_t)]\ge0,\qquad
    (q_\pi(x_t),c_\pi(x_t))\in\calC_t .
\end{equation}
The benchmark regret in the
RoS setting is
\[
    \Reg_T^{\ros}
    =
    T\OPT^{\ros}-\sum_{t=1}^T\E[r_t^{\ros}(b_t)].
\]
For a hull action \(z=(q,c)\), write \(r_t(z)=v_tq\) and
\(g_t(z)=v_tq-\kappa c\). To characterize the RoS benchmark and bound its
optimal dual multiplier, define the population dual objective
\begin{equation}\label{eq:def_overline_D}
    \overline D(\lambda)
    =
    \E_{x_t\sim\calD}
    \max_{z=(q,c)\in\calC_t}
    \{(1+\lambda)v_tq-\lambda\kappa c\}.
\end{equation}
The following strict-feasibility condition ensures strong duality for the
convexified RoS benchmark and bounds the minimizer of \(\overline D\). Note
that \(\pi_{\mathrm{sl}}\) can be any policy rather than the optimal benchmark.
\begin{assumption}[Distributional Slater condition]\label{assump:slater}
For the RoS-constrained setting, there exists a stationary policy
\(\pi_{\mathrm{sl}}\) and a constant \(\delta_S>0\) such that
\[
    \frac{1}{T}\sum_{t=1}^T
    \E\!\left[
        r_t^{\ros}(\pi_{\mathrm{sl}}(x_t))
        -\kappa c_t(\pi_{\mathrm{sl}}(x_t))
    \right]\ge\delta_S .
\]
\end{assumption}
\DIFadd{Operationally, this assumption requires the advertiser's RoS target to
be attainable with a positive safety margin. In particular, there must exist a
stationary policy whose expected value exceeds the level required by the RoS
target---for example, by concentrating bids on impressions with high expected
value relative to acquisition cost. This margin rules out targets at the
boundary of, or beyond, what the market can feasibly deliver and provides a
buffer against estimation error and temporary RoS shortfalls. Analytically,
it bounds the RoS shadow price, enabling stable projected dual updates.
Restricting the dual variable to a bounded domain is standard in online
optimization with long-term constraints}~\citep{castiglioni2022unifying}.

\section{\DIFadd{Learning Auction Competition: HoB Estimation under Full and
Binary Feedback}}
\label{sec:hob_est}

\DIFadd{The highest competing bid is the market-clearing threshold in a
first-price auction. Its conditional distribution determines a bid's win
probability and expected payment. Because this distribution is unknown, the
auto-bidder must learn auction competition from post-auction feedback. Under
full-information feedback, the bidder observes the realized HoB, whereas under
binary feedback, it observes only whether its bid wins.

This section develops feedback-specific estimators and confidence bounds for
the conditional HoB distribution. To separate competition learning from value
estimation and constrained bidding, we summarize each estimator through a
common confidence abstraction, cumulative estimation uncertainty, and the
number of HoB-exploration rounds. Subsection~}\ref{sec:hob_target_oracle}
\DIFadd{defines the estimation target and abstraction. Subsections~}
\ref{sec:full_hob_est} \DIFadd{and }\ref{sec:bin_hob_est}
\DIFadd{develop the two estimators, and Subsection~}
\ref{sec:hob_estimation_widths} \DIFadd{summarizes their effects on regret.}

\subsection{HoB Estimation Target and Oracle Abstraction}
\label{sec:hob_target_oracle}

\DIFadd{To evaluate a candidate bid \(b\), the auto-bidder needs its
conditional win probability}
\[
    F_t(b)=\Pr(m_t\le b\mid x_t)=G(b-\phi_*^\top x_t),
    \qquad b\in[0,1].
\]
\DIFadd{In a first-price auction, this quantity determines both the expected
allocation \(F_t(b)\) and the expected payment \(bF_t(b)\), and therefore
enters the constraint-adjusted bidding objective directly. Equivalently,
\(F_t\) is the conditional CDF of the HoB evaluated at the bid level.} We use
the linear model for the HoB,
\(m_t=\phi_*^\top x_t+\xi_t\), to obtain such control by estimating the
contextual shift \(\phi_*^\top x_t\) and the one-dimensional noise CDF \(G\).

The regret analysis uses HoB estimation through an abstraction that provides high-probability
confidence bounds for the conditional CDFs. We record this property in the
following abstraction.

\begin{abstraction}[HoB estimation]\label{abs:oracle_hob}
At the beginning of each round \(t\), given \(x_t\) and  \DIFaddbegin \DIFadd{the observable
history \(\mathcal H_{t-1}\)}\DIFaddend , an
oracle outputs \((\hF_t,\delta_t)\) such that, with probability at least
\(1-T^{-3}\), the confidence bound
\[
    \abs*{\hF_t(b)-F_t(b)}\le \delta_t(b)
\]
holds simultaneously for every $t\in[T]$ and every candidate bid
\DIFadd{\(b\in\calB\)}.
\end{abstraction}

Our algorithms and results apply to any HoB estimator satisfying
Abstraction~\ref{abs:oracle_hob}. For the full-information and binary-feedback
settings considered here, we construct concrete implementations based on the
following plug-in estimator:
\begin{equation}\label{eq:hob_plugin}
    \hF_t(b)=\hG_t(b-\hphi_t^\top x_t).
\end{equation}
Here \(\hphi_t\) is the time-\(t\) estimator of \(\phi_*\), so
\(\hphi_t^\top x_t\) estimates the conditional mean HoB
\(\phi_*^\top x_t\). The function \(\hG_t\) is the time-\(t\) estimator of the
time-invariant noise CDF \(G\). The CDF estimation error decomposes into two
components:
\begin{equation}\label{eq:hob_error_decomp}
\begin{aligned}
    \abs*{\hF_t(b)-F_t(b)}
    &=
    \abs*{\hG_t(b-\hphi_t^\top x_t)-G(b-\phi_*^\top x_t)}
    \\
    &\le
    \underbrace{
    \abs*{\hG_t(b-\hphi_t^\top x_t)-G(b-\hphi_t^\top x_t)}
    }_{\textnormal{noise-CDF estimation error}}
    +
    \underbrace{
    \abs*{G(b-\hphi_t^\top x_t)-G(b-\phi_*^\top x_t)}
    }_{\textnormal{shift estimation error}} .
\end{aligned}
\end{equation}
This decomposition separates the two statistical tasks needed to implement the
HoB oracle. The shift term is common to both feedback models. Since \(G\) is
\(L\)-Lipschitz,
\[
    \abs*{G(b-\hphi_t^\top x_t)-G(b-\phi_*^\top x_t)}
    \le L\abs*{(\hphi_t-\phi_*)^\top x_t}.
\]
\DIFadd{Thus, accurate estimation of \(\phi_*\) controls the error induced by
replacing the true threshold \(b-\phi_*^\top x_t\) with the plug-in threshold
\(b-\hphi_t^\top x_t\).}

The noise-CDF term is where the feedback model matters. Under full-information
feedback, observed HoBs yield predictable proxy residuals for the noise
\(\widehat{\xi}_\tau=m_\tau-\hphi_\tau^\top x_\tau\), so empirical-CDF
concentration controls this term. Under binary feedback, each win/loss
\DIFadd{observation is a Bernoulli draw whose conditional mean equals \(G\) at
the corresponding threshold. The bin-based estimator therefore pools past
observations associated with nearby estimated thresholds.}
The next two subsections instantiate these two controls
under full-information and binary HoB feedback, respectively.

\subsection{Full-Information HoB Feedback}\label{sec:full_hob_est}

Under full-information HoB feedback, the learner observes the realized HoB
\(m_\tau\) after every round. This observation can be used to estimate
both the contextual shift \(\phi_*^\top x_\tau\) and the noise distribution
\(G\), regardless of the bid submitted in that round. The natural ridge
estimator for \(\phi_*\) is
\[
    \hphi_t
    =
    \argmin_{\phi\in\R^d}
    \sum_{\tau<t}(m_\tau-\phi^\top x_\tau)^2
    +\lambda_\phi\|\phi\|_2^2 .
\]
Given \(\hphi_t\), the conditional CDF is estimated by the plug-in estimator
in \eqref{eq:hob_plugin}. Although \(m_\tau\) is observed under
full-information feedback, the realized noise
\(\xi_\tau=m_\tau-\phi_*^\top x_\tau\) is not observed because
\(\phi_*\) is unknown. We therefore form predictable proxy residuals for the noise
\[
    \widehat{\xi}_\tau=m_\tau-\hphi_\tau^\top x_\tau
\]
and define the empirical CDF of these residual proxies
\[
    \hG_t(u)
    \coloneqq
    \frac{1}{t-1}\sum_{\tau<t}\1[\widehat{\xi}_\tau\le u].
\]
Here \(\hphi_\tau\) is the estimate available before observing \(m_\tau\).
Thus the fitted value \(\hphi_\tau^\top x_\tau\) is fixed given the history
available before round \(\tau\), so \(\widehat{\xi}_\tau\) does not reuse
\(m_\tau\) in its own fitted value. \DIFadd{This predictability allows us to
apply martingale concentration despite adaptive bidding.}

\DIFadd{The following lemma shows that the plug-in estimator implements the
HoB oracle in Abstraction~}\ref{abs:oracle_hob} \DIFadd{under full-information
feedback.}

\begin{lemma}[Full-information HoB confidence band]
\label{lem:full_hob_est}
Under full-information HoB feedback, with the standing model in
\eqref{eq:induce} and the Lipschitz noise-CDF condition in
Assumption~\ref{ass:bound_density}, with probability at least \(1-T^{-3}\),
the estimator in \eqref{eq:hob_plugin} satisfies, simultaneously for all
\(t\) and \(b\in[0,1]\),
\[
    \abs*{\hF_t(b)-F_t(b)}
    \le
    \beta_{\hob}
    \parr*{
        \|x_t\|_{A_{t,\full}^{-1}}
        +\sqrt{\frac{d\log T}{t}}
    }
    \eqqcolon \delta_t ,
\]
where \(A_{t,\full}=\lambda_\phi I+\sum_{\tau<t}x_\tau x_\tau^\top\) and
\DIFadd{\(\beta_{\hob}=O(\sqrt{\log T})\), with the ridge-bias contribution
\(\sqrt{\lambda_\phi}\) and fixed model constants absorbed into the
\(O(\cdot)\) notation}.
\end{lemma}

In this bound, the \(A_{t,\full}^{-1}\)-norm term
\(\|x_t\|_{A_{t,\full}^{-1}}\) corresponds to the shift-estimation component
in \eqref{eq:hob_error_decomp}, while \(\sqrt{d\log T/t}\) corresponds to
the full-information noise-CDF estimation component.

\subsection{Binary HoB Feedback and Threshold Comparisons}
\label{sec:bin_hob_est}

Under binary HoB feedback, the learner observes only
\(o_\tau=\1[b_\tau\ge m_\tau]\). Given the history before round \(\tau\),
the current context \(x_\tau\), and the selected bid \(b_\tau\), this
observation is a Bernoulli sample with mean
\[
    \E[o_\tau\mid \mathcal H_{\tau-1},x_\tau,b_\tau]
    =\Pr(m_\tau\le b_\tau\mid
    \mathcal H_{\tau-1},x_\tau,b_\tau)
    =G(b_\tau-\phi_*^\top x_\tau).
\]
Thus each binary observation provides a noisy query of the noise CDF \(G\)
at the threshold \(b_\tau-\phi_*^\top x_\tau\).
Let \(\mathcal F_\tau^-\) contain this pre-feedback information, including the
randomization used to select \(b_\tau\). Then
\(\widehat u_\tau=b_\tau-\hphi_\tau^\top x_\tau\) is
\(\mathcal F_\tau^-\)-measurable, as is
\(\1[\widehat u_\tau\in I_j]\) for every fixed bin \(I_j\).

\DIFadd{The binary implementation uses two HoB-exploration sets that address
distinct information gaps. The set \(\mathcal E_x\) contains rounds on which
the shift test \eqref{eq:bin_hob_width_test} is triggered and improves
estimation of contextual shifts in competition. The set
\(\mathcal E_\delta\) contains rounds on which the CDF-width condition in
Line~11 of Algorithm~\ref{alg:dual-aware-lte} is satisfied and collects
information at thresholds whose local win probabilities are poorly estimated
from the binary history.} Write
\(\mathcal E_{x,t}=\mathcal E_x\cap\{1,\ldots,t-1\}\).
\DIFadd{The shift is estimated from the uniform-bid rounds in
\(\mathcal E_x\).} Since \(m_\tau\in[0,1]\),
\[
    \E[\1[b_\tau<m_\tau]\mid x_\tau]
    =\E[m_\tau\mid x_\tau]
    =\phi_*^\top x_\tau.
\]
Thus \(1-o_\tau\) is a bounded regression response with conditional mean
\(\phi_*^\top x_\tau\). We use
\begin{equation}\label{eq:binary_shift_estimator}
    A_{t,\bin}=\lambda_\phi I+
       \sum_{\tau\in\mathcal E_{x,t}}x_\tau x_\tau^\top,
    \qquad
    \hphi_t=A_{t,\bin}^{-1}
       \sum_{\tau\in\mathcal E_{x,t}}x_\tau(1-o_\tau).
\end{equation}
\DIFadd{Applying the self-normalized concentration inequality for regularized
linear regression to the bounded responses \(1-o_\tau\) yields, with
probability at least \(1-\frac{1}{2}T^{-3}\),}
\begin{equation}\label{eq:binary_shift_event}
    \abs*{(\hphi_t-\phi_*)^\top x_t}
    \le \beta_\hob\|x_t\|_{A_{t,\bin}^{-1}}
    \quad\text{for all }t,
    \qquad \beta_\hob=O(\sqrt{\log T}).
\end{equation}
Define \(a_T=d^{1/6}T^{-1/3}\). To ensure sufficient accuracy of the shift
estimate, round \(t\) is assigned to \(\mathcal E_x\) whenever
\begin{equation}\label{eq:bin_hob_width_test}
    \beta_\hob\|x_t\|_{A_{t,\bin}^{-1}}>a_T.
\end{equation}

The noise CDF is estimated by binning the predictable residual proxies
\(\widehat u_\tau\). Let \(\{I_j\}_{j=1}^J\) be ordered bins whose diameters
do not exceed \(a_T\), with representatives \(u_1<\cdots<u_J\), where
\(J=\Theta(d^{-1/6}T^{1/3})\). For each bin define
\begin{equation}\label{eq:trusted_bin_count}
    n_{t,j}=\left|\left\{\tau<t:\tau\notin\mathcal E_x,
       \widehat u_\tau\in I_j\right\}\right|,
    \qquad N_{t,j}=n_{t,j}+1.
\end{equation}
This count excludes rounds assigned to \(\mathcal E_x\) because of shift
uncertainty. In particular, CDF-width exploration rounds in
\(\mathcal E_\delta\) contribute to the count of the bin selected on that
round. For \(n_{t,j}>0\), define
\[
    \overline g_{t,j}=\frac{1}{n_{t,j}}
       \sum_{\substack{\tau<t:\,\tau\notin\mathcal E_x,\\
                       \widehat u_\tau\in I_j}}o_\tau,
\]
and set \(\overline g_{t,j}=0\) if \(n_{t,j}=0\). Define
\(\ell_T=\log(4JT^4)\) and
\(r_{t,j}=\sqrt{4\ell_T/N_{t,j}}\).

\DIFadd{The binwise averages \(\overline g_{t,j}\) provide local empirical
estimates of \(G\), but sampling and approximation errors need not preserve
monotonicity across bins. We therefore form conservative binwise lower
confidence estimates and take their cumulative maxima, following
}\citet{wen2025joint}\DIFadd{, to obtain a monotone lower estimate of \(G\):}
\[
    \underline g_{t,j}=\Pi_{[0,1]}\!\left(
       \overline g_{t,j}-r_{t,j}-2La_T\right),
    \qquad
    \widehat g_{t,j}=\max_{k\le j}\underline g_{t,k}.
\]
For \(u\in I_j\), set \(\hG_t(u)=\widehat g_{t,j}\) and
\(\hF_t(b)=\hG_t(b-\hphi_t^\top x_t)\). If \(j_t(b)\) is the bin containing
\(b-\hphi_t^\top x_t\), define
\begin{equation}\label{eq:finite_binary_hob_width}
    \delta_t(b)=\min\left\{1,
       L\beta_\hob\|x_t\|_{A_{t,\bin}^{-1}}
       +2r_{t,j_t(b)}+5La_T\right\}.
\end{equation}
Thus, even for an empty bin, \(\delta_t(b)=1\) is a valid confidence width
because both \(F_t(b)\) and \(\widehat F_t(b)\) lie in \([0,1]\).

\DIFadd{The following lemma shows that the bin-based estimator implements the
HoB oracle in Abstraction~}\ref{abs:oracle_hob} \DIFadd{under binary feedback.}

\begin{lemma}[Binary HoB confidence band]\label{lem:bin_hob_est}
\DIFadd{Under binary HoB feedback, the standing model in \eqref{eq:induce}, and
the Lipschitz noise-CDF condition in Assumption~\ref{ass:bound_density}, with
probability at least \(1-T^{-3}\), the bin-based estimator above satisfies}
\[
    \abs*{\hF_t(b)-F_t(b)}\le\delta_t(b)
\]
\DIFadd{simultaneously for every \(t\in[T]\) and \(b\in\calB\), where
\(\delta_t(b)\) is defined in \eqref{eq:finite_binary_hob_width}.}
\end{lemma}

\subsection{HoB Estimation Widths}
\label{sec:hob_estimation_widths}

\DIFadd{HoB learning affects performance through residual win-probability
uncertainty on regular bidding rounds and through rounds devoted directly to
learning auction competition. We summarize these two effects below.}

Let \(\mathcal E_\hob\) denote the rounds selected for HoB exploration and
define
\[
    \mathcal I_{\rm WLS}=[T]\setminus\mathcal E_\hob,
    \qquad
    D_{\hob,T}^2=\sum_{t\in\mathcal I_{\rm WLS}}\delta_t(b_t)^2.
\]
\DIFadd{The quantity \(D_{\hob,T}\) aggregates the HoB confidence widths along
the algorithm's realized WLS-eligible rounds.} Let \(\Delta_\hob\) denote a deterministic upper
bound for \(D_{\hob,T}\) on the oracle event.

\DIFadd{The derivations in Appendix~}\ref{app:hob_width_bounds}
\DIFadd{ yield the following bounds.}
For full HoB feedback,
\[
    \Delta_\hob=\widetilde O(\sqrt d).
\]
For binary HoB feedback,
\[
    \Delta_\hob
    =\widetilde O\!\left(\sqrt d+L d^{1/6}T^{1/6}\right).
\]
Under the standing convention \(L=O(1)\), the binary-feedback rate is
\(\widetilde O(\sqrt d+d^{1/6}T^{1/6})\). Define
\(H_\hob\coloneqq|\mathcal E_\hob|\). Lemma~\ref{lem:bounded_explore}
establishes
\[
    H_\hob
    =
    \begin{cases}
        O(d\log^2T), & \text{under full HoB feedback},\\
        \widetilde O(d^{2/3}T^{2/3}), & \text{under binary HoB feedback}.
    \end{cases}
\]
\DIFadd{Together, \(\Delta_\hob\) and \(H_\hob\) summarize the two
contributions of HoB estimation to regret: \(\Delta_\hob\) controls
cumulative estimation uncertainty on WLS-eligible rounds, whereas \(H_\hob\)
counts the HoB-exploration rounds whose contribution is bounded directly. This
separation allows the subsequent analysis to accommodate full-information and
binary HoB feedback through a common interface.}

\section{\DIFadd{Dual-Aware Learning for Constrained Auto-Bidding}}
\label{sec:dual_aware_lte_framework}

\DIFadd{This section develops the operating logic of Dual-LTE. In each round,
the auto-bidder must decide whether to learn auction competition, improve its
estimate of causal impression value, or use its current estimates to optimize
constraint-adjusted return. The budget or RoS dual state summarizes cumulative
constraint pressure through a shadow price. Dual-LTE optimizes bids when the
value and HoB estimates are sufficiently precise relative to this shadow price.
Otherwise, it directs the round toward the source of uncertainty that prevents
a reliable bidding decision.

The framework has four components. Subsection~}\ref{sec:value_est}
\DIFadd{learns uplift values from auction outcomes and quantifies value
uncertainty. Subsection~}\ref{sec:lower_conv_hull} \DIFadd{represents
randomized bids through their estimated
allocation--payment tradeoffs and identifies the empirical lower hull used for
bid selection. Subsection~}\ref{sec:better_lag}
\DIFadd{develops confidence-guided constraint-adjusted bid selection.
Subsection~}\ref{sec:dual_aware_algorithm} \DIFadd{combines these components
with the budget and RoS dual updates. Because the HoB confidence abstraction is
common across feedback regimes, the framework applies under either
full-information or binary HoB feedback and under either constraint.}

\subsection{\DIFadd{Learning Uplift Values from Auction Outcomes}}
\label{sec:value_est}

\DIFadd{The bidding decision determines which potential outcome is observed:
winning reveals the outcome with advertising, whereas losing reveals the
outcome without advertising. Thus, the observations used to estimate uplift
are selected endogenously by the bidding policy. Inverse-propensity weighting
(IPW) corrects for this selection using the estimated win probability, while
weighted least squares downweights observations associated with extreme win
probabilities and high sampling variance. Following }\citet{wen2025joint}\DIFadd{,
for a candidate bid \(b\) satisfying \(0<\hF_t(b)<1\), define}
\begin{equation}\label{eq:truncated_estimator_te}
    \te_t(b)
    =
    \frac{\1[b\ge m_t]}{\hF_t(b)}v_{t,1}
    -
    \frac{\1[b<m_t]}{1-\hF_t(b)}v_{t,0}.
\end{equation}
Let \(\sigma_t(b)\coloneqq\parr*{\hF_t(b)(1-\hF_t(b))}^{-\frac{1}{2}}\).
\DIFadd{If \(\hF_t(b)\in\{0,1\}\), we avoid evaluating the IPW ratios at a
degenerate estimated propensity by assigning the observation zero WLS weight,
with the convention}
\[
    \DIFadd{\sigma_t(b)^{-2}=0,
    \qquad
    \sigma_t(b)^{-4}\te_t(b)=0.
}\]
The Hoeffding-type HoB oracle in Abstraction~\ref{abs:oracle_hob} \DIFadd{implies} the
following bias and variance bounds for \(\te_t(b)\). The proof is given in
Appendix~\ref{app:estimation}.

\begin{lemma}[IPW bias and variance]
\label{lem:bias-variance}
Suppose Abstraction~\ref{abs:oracle_hob} holds. In either \(\full\) or
\(\bin\) mode, \DIFaddbegin \DIFadd{for every \(b\) satisfying \(0<\hF_t(b)<1\),
}\DIFaddend \[
\begin{cases}
    \abs*{\E[\te_t(b)]-\theta_*^\top x_t}
    \le
    2\delta_t(b)\sigma_t(b)^2,\\
    \Var(\te_t(b))
    \le
    4\sigma_t(b)^4 .
\end{cases}
\]
\end{lemma}

\DIFadd{Let \(\mathcal E_\hob\) denote the HoB-exploration rounds selected by
Algorithm~}\ref{alg:dual-aware-lte}\DIFadd{. Specifically,}
\[
\DIFadd{
    \mathcal E_\hob
    =
    \begin{cases}
        \mathcal E_\delta, & \text{under full-information feedback},\\
        \mathcal E_x\cup\mathcal E_\delta, & \text{under binary feedback}.
    \end{cases}
}
\]
\DIFadd{These rounds are used to update the HoB estimator but are excluded from
treatment-effect estimation.}
Define
\[
    \mathcal I_{{\rm WLS},t}=\{\tau<t:\tau\notin\mathcal E_\hob\}.
\]
On value-exploration rounds in \(\mathcal I_{{\rm WLS},t}\), the learner draws
a binary bid \(b_t\in\{0,1\}\) uniformly.
Since \(m_t\in[0,1]\), this samples treatment and control outcomes with fixed
probability \(1/2\). We set
\[
    \te_t
    =
    2\parr*{\1[b_t=1]v_{t,1}-\1[b_t=0]v_{t,0}},
    \qquad
    \sigma_t=1
\]
on these rounds, which is unbiased and satisfies
$\Var(\te_t)\le 4\sigma_t^4$. On other WLS-eligible rounds \DIFaddbegin \DIFadd{with
\(0<\hF_t(b_t)<1\)}\DIFaddend , set
\(\te_t=\te_t(b_t)\) and \(\sigma_t=\sigma_t(b_t)\). \DIFaddbegin \DIFadd{If
\(\hF_t(b_t)\in\{0,1\}\), use the zero-weight convention above. }\DIFaddend The weighted
least-squares estimator is
\begin{equation}\label{eq:wls_defs}
    A_t
    =
    I+\sum_{\tau\in\mathcal I_{{\rm WLS},t}}
       \sigma_\tau^{-4}x_\tau x_\tau^\top,
    \qquad
    y_t
    =
    \sum_{\tau\in\mathcal I_{{\rm WLS},t}}
       \sigma_\tau^{-4}\te_\tau x_\tau,
    \qquad
    \htheta_t=A_t^{-1}y_t.
\end{equation}
Equivalently, \(\htheta_t\) is the ridge-regression solution:
\[
    \htheta_t
    \in
    \argmin_{\theta\in\mathbb{R}^d}
    \left\{
        \|\theta\|_2^2
        +
        \sum_{\tau\in\mathcal I_{{\rm WLS},t}}
        \sigma_\tau^{-4}
        \bigl(\te_\tau-\theta^\top x_\tau\bigr)^2
    \right\}.
\]
\DIFadd{Since Lemma~}\ref{lem:bias-variance}\DIFadd{ bounds the conditional
variance by \(4\sigma_\tau^4\), we use the inverse-variance weight
\(\sigma_\tau^{-4}\). This weighting downweights high-variance IPW
observations and places the weighted noise terms on a comparable variance
scale.}
\DIFadd{This weighting also aligns the WLS geometry with the allocation factor
appearing in the later Lagrangian analysis.} If \(p=\hF_t(b_t)\), then
\begin{equation}\label{eq:allocation_variance_scale}
    p(1-p)\le \min\{p,1-p\}\le 2p(1-p),
\end{equation}
\DIFadd{so the allocation factor multiplying the value radius is within a
constant factor of \(\sigma_t(b_t)^{-2}=p(1-p)\). Thus, the transformed feature
\(\sigma_t(b_t)^{-2}x_t\) matches the geometry of the weighted design matrix
\(A_t\).}

The next lemma gives the value confidence bound used by the branching and
Lagrangian analysis.

\begin{lemma}[Weighted value confidence bound]
\label{lem:wls}
Suppose Abstraction~\ref{abs:oracle_hob} holds. With probability at least
\(1-\delta\), for all \(t\),
\[
    \|\htheta_t-\theta_*\|_{A_t}
    \le
    \beta_t
    \coloneqq
    1+C\sqrt{d\log(T/\delta)}
    +C\sqrt{\sum_{\tau\in\mathcal I_{{\rm WLS},t}}
       \delta_\tau(b_\tau)^2}
\]
for some absolute constant \(C>0\). Consequently,
\begin{equation}\label{eq:value_wls_width}
    \abs*{\htheta_t^\top x_t-\theta_*^\top x_t}
    \le
    \beta_t\|x_t\|_{A_t^{-1}}
    \eqqcolon R_t .
\end{equation}
Here \(R_t\) denotes the value-confidence radius at context \(x_t\). By the
definition of \(\Delta_\hob\),
\[
    \beta_t
    \le 1+C\sqrt{d\log(T/\delta)}+C\Delta_\hob.
\]
\end{lemma}

\subsection{\DIFadd{Allocation--Payment Representation and Empirical Lower
Hull}}
\label{sec:lower_conv_hull}

\DIFadd{For both constraints, the per-round Lagrangian depends on a bid only
through its allocation probability and expected payment. We therefore
represent bids by allocation--payment pairs and randomized bid mixtures by
convex combinations of these pairs. For the budget constraint, this
representation is primarily an analytical convenience. For the RoS constraint,
it also captures randomized mixtures that may improve the attainable aggregate
value--spend tradeoff.}

\DIFadd{Because \(F_t\) is unknown, the learner represents bid \(b\) by the
estimated pair \((\hF_t(b),b\hF_t(b))\). Because a selected hull action must
ultimately be implemented using actual bids, we attach the generating bid to
each estimated allocation--payment point and define}
\[
    \hcalC_t
    \coloneqq
    \mathrm{conv}\parr*{
        \bra*{(\hF_t(b),b\hF_t(b),b):b\in\calB}
    } .
\]
\DIFadd{The first two coordinates record the estimated allocation probability
and expected payment, while the third identifies the corresponding bid.
Because different bids and their mixtures may attain the same estimated
allocation probability at different expected payments, and both
constraint-adjusted objectives weakly prefer the lower payment, we restrict
attention to the empirical lower hull. Ties in estimated allocation
probability are resolved in favor of the smaller bid.} For any bid subset
\(S\subseteq\calB\), let \(\Clow_t(S)\) denote the empirical lower hull
constructed from
\[
    \bra*{(\hF_t(b),b\hF_t(b),b):b\in S}.
\]
We write \(\Clow_t=\Clow_t(\calB)\) for the full-grid lower hull.

\DIFadd{For comparison, define the population grid hull}
\[
    \DIFadd{\calC_t^\star
    \coloneqq
    \mathrm{conv}\{(F_t(b),bF_t(b)):b\in\calB\}.
}
\]
\DIFadd{The two objects play different roles: \(\calC_t^\star\) is constructed
from the true CDF and defines the population Lagrangian benchmark, whereas
\(\Clow_t(S)\) is constructed from the estimated CDF over a bid subset
\(S\subseteq\calB\) and is used by the learner. Their connection is bidwise:
if an empirical lower-hull action randomizes between bids \(b^-,b^+\in S\),
then applying the same mixing weights to their true allocation--payment points
defines its population counterpart in \(\calC_t^\star\).}

\DIFadd{The following fact shows that any lower-hull action can be implemented
by randomizing between the two bids associated with its adjacent lower-hull
vertices.}

\begin{fact}\label{fact:two_vertex_decompose}
For any \(S\subseteq\calB\), every action \(z\in \Clow_t(S)\) on a lower-hull
edge can be represented as a convex combination of two adjacent lower-hull
vertices:
\[
    z=\alpha z^-+(1-\alpha)z^+,\qquad \alpha\in[0,1],
\]
where \(z^-=(\hF_t(b^-),b^-\hF_t(b^-),b^-)\) and
\(z^+=(\hF_t(b^+),b^+\hF_t(b^+),b^+)\) are adjacent lower-hull vertices.
\end{fact}

Accordingly, after selecting \(z\in \Clow_t(S)\) for some
\(S\subseteq\calB\), the algorithm implements this convex combination by
bidding \(b^-\) with probability \(\alpha\) and \(b^+\) with probability
\(1-\alpha\).
For such a lower-hull action \(z\), define its estimated allocation and payment by
\begin{equation}\label{eq:def_allocation_and_payment}
    \hq_t(z)
    \coloneqq
    \alpha\hF_t(b^-)+(1-\alpha)\hF_t(b^+),
    \qquad
    \hc_t(z)
    \coloneqq
    \alpha b^-\hF_t(b^-)+(1-\alpha)b^+\hF_t(b^+).
\end{equation}
The corresponding population allocation and payment are
\begin{equation}\label{eq:def_pop_allocation_and_payment}
    q_t(z)
    \coloneqq
    \alpha F_t(b^-)+(1-\alpha)F_t(b^+),
    \qquad
    c_t(z)
    \coloneqq
    \alpha b^-F_t(b^-)+(1-\alpha)b^+F_t(b^+).
\end{equation}
\DIFadd{The estimated quantities \(\hq_t(z)\) and \(\hc_t(z)\) enter the
empirical Lagrangian scores and confidence widths in the next subsection.}

\subsection{\DIFadd{Confidence-Guided Lagrangian Bid Selection}}
\label{sec:better_lag}
\label{sec:two_lag}

\DIFadd{HoB exploration and value exploration address different sources of
uncertainty. HoB-exploration rounds improve the estimate of the conditional win
probability \(F_t\), but they do not by themselves provide sufficiently
informative treatment and control outcomes for estimating the causal uplift
value \(v_t\). When the value-confidence radius \(R_t\) remains large, the
algorithm therefore randomizes equally between the endpoint bids \(0\) and
\(1\), thereby observing the control and treatment outcomes with equal
probability and improving the value estimate. The same value-exploration rule
is used when the algorithm cannot certify that the current estimates support
reliable Lagrangian optimization.}

\DIFadd{A constraint-adjusted bid can be unreliable even when its estimated
Lagrangian is optimistic because value uncertainty is amplified differently
for bids with low and high win probabilities. We therefore adapt the
``better of two UCBs'' idea of }\citet{wen2025joint} \DIFadd{to two equivalent
constraint-adjusted Lagrangian representations. Their confidence widths have
value-error coefficients \(\hq_t(z)\) and \(1-\hq_t(z)\), which expose the two
allocation regimes. Under constraints, the regime containing the population
optimizer depends on the current shadow price, so the representation cannot be
selected solely by comparing the estimated scores. The branching routine first
identifies a candidate set containing the population optimizer and then uses
the representation whose value-error coefficient is controlled by the
allocation variance throughout that set. If neither regime can be identified
reliably, the algorithm explores instead.}

For a given shadow price \(\rho_t\), let
\(\widehat v_t=\htheta_t^\top x_t\) be the current value estimate, and let
\(\hq_t(z)\) and \(\hc_t(z)\) be the estimated allocation and payment of the
lower-hull action \(z\), as defined in \eqref{eq:def_allocation_and_payment}
using the two-vertex decomposition in Fact~\ref{fact:two_vertex_decompose}.
We define the empirical Lagrangians
\begin{equation}\label{eq:def_two_lag}
\begin{split}
    \hl_t^0(z;\rho_t)
    &\coloneqq
    \hq_t(z)\widehat v_t-\rho_t\hc_t(z),\\
    \hl_t^1(z;\rho_t)
    &\coloneqq
    -(1-\hq_t(z))\widehat v_t-\rho_t\hc_t(z)
    =
    \hl_t^0(z;\rho_t)-\widehat v_t .
\end{split}
\end{equation}
\DIFadd{The two functions have the same maximizers because they differ only by
the action-independent term \(\widehat v_t\).} For a
randomized lower-hull action \(z=\alpha z^-+(1-\alpha)z^+\), where \(z^-\) and
\(z^+\) are generated by bids \(b^-\) and \(b^+\), respectively, define the
corresponding confidence widths by
\begin{equation}\label{eq:def_two_lag_widths}
\begin{split}
    w_t^0(z)
    &=
    \hq_t(z)R_t
    +(1+\rho_t)\parr*{\alpha\delta_t(b^-)+(1-\alpha)\delta_t(b^+)},\\
    w_t^1(z)
    &=
    (1-\hq_t(z))R_t
    +(1+\rho_t)\parr*{\alpha\delta_t(b^-)+(1-\alpha)\delta_t(b^+)} .
\end{split}
\end{equation}
The next lemma combines the value-confidence bound in
Lemma~\ref{lem:wls} with the HoB confidence bound in
Abstraction~\ref{abs:oracle_hob} to establish valid confidence widths for both
shifted Lagrangian functions.

\begin{lemma}[Lagrangian widths]
\label{lem:two_lag_width_valid}
Suppose Abstraction~\ref{abs:oracle_hob} and the event in
Lemma~\ref{lem:wls} hold. For every hull action \(z\in\Clow_t\),
\[
    \abs*{
        \hl_t^0(z;\rho_t)
        -
        \parr*{q_t(z)v_t-\rho_t c_t(z)}
    }
    \le
    w_t^0(z),
\]
and
\[
    \abs*{
        \hl_t^1(z;\rho_t)
        -
        \parr*{-(1-q_t(z))v_t-\rho_t c_t(z)}
    }
    \le
    w_t^1(z).
\]
\end{lemma}

\begingroup
\SingleSpacedXI
\begin{algorithm}[!t]
\caption{CDF-interval localization routine}
\label{alg:ucb_selection}
\begin{algorithmic}[1]
\State \textbf{Input:} perturbation \(c>0\), estimated CDF \(\hF_t\),
estimated value \(\widehat v\), shadow price \(\rho_t\), bid grid \(\calB\).
\State Set
\(b_{+}\gets\argmax_{b\in\calB}
\hF_t(b)(\widehat v+2\rho_tc-\rho_tb)\).
\State Set
\(b_{-}\gets\argmax_{b\in\calB}
\hF_t(b)(\widehat v-2\rho_tc-\rho_tb)\).
\State Set
\(U\gets\bra*{b\in\calB:
\hF_t(b)\in[\hF_t(b_-)-c,\hF_t(b_+)+c]}\).
\State Set \(b_{\mathrm{left}}\gets \min U\) and
\(b_{\mathrm{right}}\gets \max U\).
\State Output interval \((b_{\mathrm{left}},b_{\mathrm{right}})\) and
\(i=\1[\hF_t(b_{\mathrm{left}})\ge c]\).
\end{algorithmic}
\end{algorithm}
\endgroup

\DIFadd{The two widths in \eqref{eq:def_two_lag_widths} differ in their
value-error coefficients, \(\hq_t(z)\) and \(1-\hq_t(z)\). The branching
routine uses two tests to ensure that the selected coefficient is controlled by
the allocation-variance factor \(\hq_t(z)(1-\hq_t(z))\). Under the high-value
condition \(\widehat v_t\ge R_t+\rho_t/\zeta\), the population-optimal bid
belongs to
\(\calB_t=\{b:\hF_t(b)\ge1-\zeta-\Delta_t\}\). This candidate set contains
only bids with high estimated win probabilities, and this property is
preserved under bid randomization. The routine therefore uses \(\hl_t^1\),
whose associated confidence width has the smaller value-uncertainty term in
this region. Under the small-estimation-error condition}
\[
    R_t+\Delta_t(\widehat v_t+R_t+\rho_t)\le c^2\rho_t,
\]
\DIFadd{the perturbation routine localizes the population optimizer and selects
the shifted Lagrangian whose value-error coefficient is controlled by the
allocation-variance factor throughout the resulting candidate set. If neither
condition holds, the routine sets \(u_t=1\) and uses the round for exploration.}

\begingroup
\SingleSpacedXI
\begin{algorithm}[!t]
\caption{Lagrangian branching and exploration routine}
\label{alg:ucb_branching}
\begin{algorithmic}[1]
\State \textbf{Input:} perturbation \(c\in(0,1/4)\), margin
\(\zeta\in(0,1/4)\), estimated CDF \(\hF_t\), CDF error bound
\(\Delta_t\in(0,\zeta/2)\), estimated value \(\widehat v_t\), value width
\(R_t\), shadow price \(\rho_t\), bid grid \(\calB\).
\State Initialize exploration flag \(u_t\gets0\).
\If{\(\widehat v_t\ge R_t+\rho_t/\zeta\)}
    \State Set \(\calB_t\gets
    \{b\in\calB:\hF_t(b)\ge 1-\zeta-\Delta_t\}\).
    \State Set \(i_t\gets1\).
\ElsIf{\(R_t+\Delta_t(\widehat v_t+R_t+\rho_t)\le c^2\rho_t\)}
    \State Invoke Algorithm~\ref{alg:ucb_selection} with
    \((c,\hF_t,\widehat v_t,\rho_t,\calB)\) and receive the interval
    \((b_{\mathrm{left}},b_{\mathrm{right}})\) and index \(i_t\).
    \State Set \(\calB_t\gets
    (b_{\mathrm{left}},b_{\mathrm{right}})\cap\calB\).
\Else
    \State Set \(u_t\gets1\).
\EndIf
\State If \(u_t=0\), return \(\calB_t\), \(i_t\), and \(u_t\).
Otherwise, return \(u_t=1\).
\end{algorithmic}
\end{algorithm}
\endgroup

\DIFadd{The next lemma shows that the branching routine preserves the
population optimizer while selecting a value-error coefficient that is
compatible with the weighted-design geometry in
\eqref{eq:allocation_variance_scale}. It also bounds the loss incurred when the
routine explores instead.}

\begin{lemma}[Returned candidate set contains optimal hull actions]
\label{lem:branching_control}
Fix a round \(t\) and the inputs of Algorithm~\ref{alg:ucb_branching}. If the
routine returns \(u_t=0\), let \(\calB_t\subseteq\calB\) denote the
round-dependent candidate bid set returned by the routine. This set is certified
to contain the population optimizer within the fixed grid \(\calB\). Then every
maximizer of
\(
    (q,c)\mapsto v_tq-\rho_tc
\)
over \(\calC_t^\star\)
lies in
\(
    \mathrm{conv}\{(F_t(b),bF_t(b)):b\in\calB_t\}.
\)
Moreover, for all \(b\in\calB_t\),
\[
\begin{cases}
    \hF_t(b)
    \le
    20L\hF_t(b)(1-\hF_t(b)), &\textnormal{if }i_t=0,\\
    1-\hF_t(b)
    \le
    20L\hF_t(b)(1-\hF_t(b)), &\textnormal{if }i_t=1.
\end{cases}
\]
\DIFadd{If \(u_t=1\), the resulting one-step Lagrangian loss is \(O(R_t)\).}
\end{lemma}

\subsection{\DIFadd{Dual-LTE Algorithm and Constraint Updates}}
\label{sec:dual_aware_algorithm}

\DIFadd{Algorithm~}\ref{alg:dual-aware-lte} \DIFadd{combines the HoB and
uplift-value estimators with either constraint. Operationally, it follows a
hierarchy of information needs. It first resolves material uncertainty about
auction competition and then checks whether the uplift estimate is sufficiently
accurate for the current shadow price. If additional value information is
needed, it uses balanced endpoint exploration. Otherwise, it selects a
constraint-adjusted action from the empirical lower hull. The resulting
feedback updates the estimators and constraint state for the next round.

The full-information and binary HoB procedures are given in
Section~}\ref{sec:hob_est}\DIFadd{. In Algorithm~}\ref{alg:dual-aware-lte}
\DIFadd{, the binary shift-uncertainty and CDF-width checks appear in Lines 8
and 11, respectively. The budget or RoS update maps its dual state to the
shadow price \(\rho_t\), which parameterizes the shared bid-selection rule.}

The RoS multiplier and budget pacing factor are initialized in
Algorithm~\ref{alg:dual-aware-lte} as
\begin{equation}\label{eq:dual_state_init}
    \lambda_1=1\ \text{under RoS},\qquad
    \mu_1=1\ \text{under Budget.}
\end{equation}
Let $\eta\ge 0$ be an input learning rate. The dual states are computed via mirror-descent-based update rules below.

\begingroup
\SingleSpacedXI
\begin{algorithm}[!t]
\caption{Dual-Aware LTE framework for repeated FPAs}
\label{alg:dual-aware-lte}
\begin{algorithmic}[1]
\State \textbf{Input:} horizon \(T\), feedback module
\(\mathrm{FB}\in\{\full,\bin\}\), constraint module
\(\mathrm{CON}\in\{\bgt,\ros\}\), HoB oracle \(\mathcal O\), bid grid
\(\calB\), pacing rate $\eta>0$.
\State Initialize the WLS state $A_1=I, y_1=\htheta_1=0$.
\State Initialize \(\mathcal E_x=\mathcal E_\delta=\mathcal E_\hob=\varnothing\)
and all binary trustworthy counts \(n_{1,j}=0\).

\State Initialize the dual state of \(\mathrm{CON}\) by \eqref{eq:dual_state_init}, and set \(\rho_1\) by the corresponding constraint module.
\For{\(t=1,\ldots,T\)}
    \State Observe \(x_t\), set \(\widehat v_t=\htheta_t^\top x_t\), and query
    \(\mathcal O\) for \((\hF_t,\delta_t)\).
    \State Compute value width \(R_t\) from \eqref{eq:value_wls_width}.

    \If{under
    \(\mathrm{FB}=\bin\), \eqref{eq:bin_hob_width_test} holds}
        \State Add \(t\) to \(\mathcal E_x\) and \(\mathcal E_\hob\).
        \State Draw the HoB-exploration bid \(b_t\sim\mathrm{Unif}([0,1])\).
     \DIFaddbegin \ElsIf{\(\max_{b\in\calB}\delta_t(b) > \frac{1}{9600L^2}\)}
        \DIFaddend \State Add \(t\) to \(\mathcal E_\delta\) and \(\mathcal E_\hob\), and draw
        \(b_t\in\argmax_{b\in\calB}\delta_t(b)\). Set
        \(j_t=j_t(b_t)\).
    \ElsIf{$R_t > (d^{\frac{1}{2}} + \Delta_\hob^{\frac{1}{2}})T^{-\frac{1}{4}}$}
        \State Draw the value-exploration bid \(b_t\sim\mathrm{Bern}(\frac{1}{2})\).
    \Else
        \State Run Algorithm~\ref{alg:ucb_branching} with $c=\frac{1}{20L}$ and $\zeta=\frac{1}{8}$ to obtain
        \((\calB_t,i_t,u_t)\).
        \If{\(u_t=1\)}
            \State Draw the value-exploration bid \(b_t\sim\mathrm{Bern}(\frac{1}{2})\).
        \Else
            \State Set \(z_t\in\argmax_{z\in\Clow_t(\calB_t)}
            \{\hl_t^{i_t}(z;\rho_t)+w_t^{i_t}(z)\}\).
            \State Sample $b_t$ from the induced two-bid randomization from $z_t$ (c.f. Fact \ref{fact:two_vertex_decompose}).
        \EndIf
    \EndIf

    \State Submit \(b_t\) and observe the available feedback.
    \State Under \(\mathrm{FB}=\bin\): if \(t\in\mathcal E_x\), update the
    shift state in \eqref{eq:binary_shift_estimator}. Otherwise, use the
    estimated threshold \(\widehat u_t\) to assign \(o_t\) to a bin and
    increment that bin's count.
    \If{\(t\notin\mathcal E_\hob\)}
        \State Construct \((\te_t,\sigma_t)\) as in Section~\ref{sec:value_est}
        and update \((A_{t+1},y_{t+1},\htheta_{t+1})\) by
        \eqref{eq:wls_defs}.
    \Else
        \State Set \((A_{t+1},y_{t+1},\htheta_{t+1})=(A_t,y_t,\htheta_t)\).
    \EndIf

    \State Update the constraint-module dual state and $\rho_{t+1}$
    using \eqref{eq:budget_module} under \(\bgt\) or \eqref{eq:ros_module} under \(\ros\).

    \If{$\mathrm{CON}=\bgt$ and budget depletes}
    \State Break the loop and stop.
    \EndIf
\EndFor
\end{algorithmic}
\end{algorithm}
\endgroup

\smallskip
\noindent\textit{Budget constraint.}\quad
In the budget setting, at the end of time \(t\) in
Algorithm~\ref{alg:dual-aware-lte}, we set
\begin{equation}\label{eq:budget_module}
    \mu_{t+1}
    =
    \Pi_{[0,1]}\mu_t\exp\parr*{\eta\parr*{c_t-\frac{B}{T}}},\qquad
    \rho_{t+1}=1+\sqrt{\frac{LT}{B}}\mu_{t+1},
\end{equation}
where $c_t = \1[b_t\ge m_t]b_t$ denotes the realized spend at time $t$.
\DIFadd{The pacing update increases the multiplier when realized spend exceeds
the per-round spending target \(B/T\) and decreases it when realized spend
falls below that target.}

\smallskip
\noindent\textit{RoS constraint.}\quad
\DIFadd{In the RoS setting, we require an input multiplier bound $\Lambda>0$ and restrict $\lambda_t$ to $[0,\Lambda]$. Define an upper confidence bound (UCB) for the RoS slack by}
\begin{equation}\label{eq:slack_ucb}
\bg_t(b) = \Pi_{[-\kappa,1]}\parr*{\hF_t(b)(\hv_t - \kappa b) + R_t + (1+\kappa)\delta_t(b)} \ge g_t(b) = F_t(b)(v_t-\kappa b).
\end{equation}
We update the dual states by
\begin{equation}\label{eq:ros_module}
    \lambda_{t+1}
    =
    \Pi_{[0,\Lambda]}\lambda_t\exp(-\eta\bg_t(b_t)),\qquad \rho_{t+1} = \frac{\kappa\lambda_{t+1}}{1+\lambda_{t+1}}.
\end{equation}
\DIFadd{When the estimated RoS slack is negative, the update increases the
multiplier and shadow price, making subsequent bidding more conservative.
Positive slack reduces constraint pressure and permits more aggressive value
acquisition.}
The following lemma shows that \([0,\Lambda]\) contains every
population-optimal RoS multiplier whenever \(\Lambda\ge 1/\delta_S\).
\begin{lemma}[Valid lambda projection]\label{lem:lambda_bound}
Suppose Assumption \ref{assump:slater} holds. Then every minimizer $\lambda_*\in\argmin_{\lambda\ge 0}\overline{D}(\lambda)$ satisfies $\lambda_* \le \frac{1}{\delta_S}$, where $\overline{D}$ is defined in \eqref{eq:def_overline_D}.
\end{lemma}

\begin{remark}[Choice of the RoS projection bound]
\DIFadd{The RoS dual update in }\eqref{eq:ros_module}\DIFadd{ projects the
multiplier onto \([0,\Lambda]\). Lemma~}\ref{lem:lambda_bound}\DIFadd{ shows
that \(\lambda_*\le 1/\delta_S\), while the regret and violation guarantees
below use}
\[
    \DIFadd{\Lambda\ge\max\left\{1,\frac{2}{\delta_S}\right\}.
}
\]
\DIFadd{If \(\delta_S\) is unknown, a conservative lower confidence bound
\(\underline\delta_S\le\delta_S\) may be obtained from offline data or an
online calibration procedure, and one may set}
\[
    \DIFadd{\Lambda
    =
    \max\left\{1,\frac{2}{\underline\delta_S}\right\}.
}
\]
\end{remark}


\section{Regret Guarantees and Instantiations}
\label{sec:regret_instantiations}


\DIFadd{This section presents the end-to-end performance guarantees for
Dual-LTE. Subsection~}\ref{sec:reg_vio_budget} \DIFadd{establishes regret under
the budget constraint, while Subsection~}\ref{sec:reg_vio_ros}
\DIFadd{establishes regret and violation guarantees under the RoS constraint.
After substituting the feedback-specific HoB bounds, the regret rates scale as
\(\widetilde O(\sqrt T)\) under full-information feedback and
\(\widetilde O(T^{2/3})\) under binary feedback when non-horizon problem
parameters are fixed. These rates match the optimal horizon dependence
established by }\citet{wen2025joint} \DIFadd{for unconstrained joint value and
HoB learning. Thus, incorporating budget or RoS constraints preserves the
known optimal regret exponents in \(T\).}

\subsection{Regret under Budget Constraint}
\label{sec:reg_vio_budget}

In this setting, Algorithm \ref{alg:dual-aware-lte} stops bidding whenever the hard budget constraint is met (Line 24). The pacing is controlled by comparing the realized cost against the averaged budget $\frac{B}{T}$ in \eqref{eq:budget_module}. The following result summarizes the performance of Algorithm \ref{alg:dual-aware-lte}.

\begin{theorem}[Budget regret]\label{thm:budget_regret}
Suppose Assumptions~\ref{ass:unconfound}--\ref{ass:linear_budget} hold,
and suppose the HoB oracle  \DIFaddbegin \DIFadd{satisfies
}\DIFaddend Abstraction~\ref{abs:oracle_hob} \DIFaddbegin \DIFadd{. Use confidence level
\(\delta=T^{-3}\) in the value radius of Lemma~\ref{lem:wls}.
}\DIFaddend Run the Budget instantiation of
Algorithm~\ref{alg:dual-aware-lte} with an evenly discretized bid grid
\(\calB\subseteq[0,1]\) of size \(K\) and
learning rate \(\eta=\Theta(T^{-\frac{1}{2}})\). Then
\[
    \Reg_T^{\bgt}
    =
    O\parr*{
        d\sqrt T\log^2T
        +\Delta_{\hob}\sqrt{dT}\log T
        +H_\hob
        + \frac{T}{K}
    } .
\]
\end{theorem}

The full proofs of Theorem~\ref{thm:budget_regret} and
Corollary~\ref{cor:budget_feedback} are deferred to
Appendix~\ref{app:budget_proofs}, and we provide a proof sketch below.

\smallskip
\noindent\textit{Proof sketch of Theorem~\ref{thm:budget_regret}.}\quad
Let \(\tau\) be the last round before the hard budget is depleted. With
\(Z=\sqrt{LT/B}\), the budget shadow price is \(\rho_t=1+Z\mu_t\). The
regret is decomposed into a Lagrangian-learning term, a pacing term, and the
loss after stopping:
\begin{align*}
\Reg_T^{\bgt}
&\le
\E\!\left[\sum_{t=1}^{\tau}
    \ell_t^{0}(z_t^*;\rho_t)-\ell_t^{0}(z_t;\rho_t)\right]
+Z\E\!\left[\sum_{t=1}^{\tau}
    \mu_t\left(\frac{B}{T}-\bc_t(b_t)\right)\right]
+\E\!\left[\sum_{t=\tau+1}^{T} r_t^{\bgt}(b_t^*)\right].
\end{align*}
The first term is controlled by the dual-aware LTE learning bound: the
better-of-two Lagrangian construction pairs value-estimation error with the
allocation variance, while the HoB oracle contributes the cumulative width
\(\Delta_{\hob}\). This gives
\(O(d\sqrt T\log^2T+\Delta_{\hob}\sqrt{dT}\log T+H_\hob+T/K)\), where the
contribution of HoB-exploration rounds is accounted for once through
Lemma~\ref{lem:bounded_explore}. The pacing term is
controlled by the entropic mirror-descent update for \(\mu_t\) with feedback
\(B/T-c_t\). Under Assumption~\ref{ass:linear_budget}, \(B=\Theta(T)\), so
the resulting \(O(\sqrt T)\) contribution is dominated by the learning term.
\DIFadd{If \(\tau<T\), the stopping rule implies that the realized cumulative spend
differs from the budget by at most one,} and the same pacing argument bounds the remaining comparator
reward. The \(T/K\) term comes from approximating the continuous benchmark bid
by the evenly discretized grid.

\begin{corollary}[Budget rates under HoB feedback]\label{cor:budget_feedback}
Under the conditions of Theorem~\ref{thm:budget_regret} and discretization grid $K=\Omega(\sqrt{T})$, the full-information
HoB instantiation in Section \ref{sec:full_hob_est} satisfies
\[
    \Reg_T^{\bgt,\full}
    =
    \widetilde O\parr*{d\sqrt T} .
\]
With $K=\Omega(T^{\frac{1}{3}})$, the binary HoB instantiation in Section \ref{sec:bin_hob_est}
satisfies
\[
    \Reg_T^{\bgt,\bin}
    =
    \widetilde O\parr*{
        d\sqrt T+d^{\frac{2}{3}}T^{\frac{2}{3}}
    } .
\]
In particular, for fixed \(d\), these rates are
\(\widetilde O(\sqrt T)\) and \(\widetilde O(T^{\frac{2}{3}})\), respectively.
\end{corollary}

\subsection{Regret and Violation under RoS Constraint}\label{sec:reg_vio_ros}

In the RoS setting, the bidder optimizes the per-round reward \(r_t^\ros\)
subject to the long-term return-on-spend target \(\kappa\). We establish a
regret guarantee for Algorithm~\ref{alg:dual-aware-lte} together with a
sublinear bound on the cumulative violation \(\Viol_T^\ros\), which implies
near-optimal performance with vanishing average RoS violation.

\begin{theorem}[RoS regret and violation]\label{thm:ros_regret_violation}
Suppose Assumptions~\ref{ass:unconfound}--\ref{ass:bound_density} and
\ref{assump:slater} hold, and suppose the HoB oracle  \DIFaddbegin \DIFadd{satisfies
}\DIFaddend Abstraction~\ref{abs:oracle_hob} \DIFaddbegin \DIFadd{. Use confidence level
\(\delta=T^{-3}\) in the value radius of Lemma~\ref{lem:wls}. }\DIFaddend Run the RoS instantiation of
Algorithm~\ref{alg:dual-aware-lte} with an
evenly discretized bid grid \(\calB\subseteq[0,1]\) of size \(K\),
\(\Lambda\ge\max\{1,\frac{2}{\delta_S}\}\), and
\(\eta=\Theta(T^{-\frac{1}{2}})\). Then
\[
\begin{aligned}
    \Reg_T^{\ros}
    &=
    O\parr*{
        \Lambda\parr*{
        d\sqrt T\log^2T
        +\Delta_{\hob}\sqrt{dT}\log T
        +H_\hob
        +\frac{T}{K}}
    },
    \\
    \Viol_T^{\ros}
    &=
    O\parr*{
        d\sqrt T\log^2T
        +\Delta_{\hob}\sqrt{dT}\log T
        +H_\hob
        +
        \parr*{d^{\frac{1}{2}}
        +\Delta_{\hob}^{1/2}}T^{\frac{3}{4}}
        +\frac{T}{K}
    } .
\end{aligned}
\]
\end{theorem}

The full proofs of Theorem~\ref{thm:ros_regret_violation} and
Corollary~\ref{cor:ros_feedback} are deferred to
Appendix~\ref{app:ros_proofs}, and we provide a proof sketch below.

\smallskip
\noindent\textit{Proof sketch of Theorem~\ref{thm:ros_regret_violation}.}\quad
The convexified benchmark makes the RoS-constrained problem linear in the
allocation--payment action. By the Slater condition and
Lemma~\ref{lem:lambda_bound}, strong duality holds and it suffices to project
the multiplier onto \([0,\Lambda]\). Under strong duality, the regret becomes
\[
\Reg_T^{\ros}  \DIFaddbegin \le \DIFaddend \E\!\left[
    \sum_{t=1}^T
    \DIFaddbegin \left\{\DIFaddend (1+\lambda_t)\parr*{\max_{z\in\calC_t}\ell_t^0(z;\rho_t)
    -\ell_t^0(z_t;\rho_t)}+\lambda_tg_t\DIFaddbegin \DIFadd{(b_t)}\right\}
    \DIFaddend \right] \DIFaddbegin \DIFadd{.
}\DIFaddend \]
where $\calP_T$ denotes the Lagrangian-learning term
\[
    \calP_T
    =
    \E\!\left[
    \sum_{t=1}^T
    (1+\lambda_t) \DIFaddbegin \parr*{\max_{z\in\calC_t}\ell_t^0(z;\rho_t)
    -\ell_t^0(z_t;\rho_t)}
    \DIFaddend \right].
\]
By a mirror-descent argument and $\lambda_t\le \Lambda$, $\sum_{t=1}^T\lambda_t\overline g_t=O(\Lambda\sqrt{T})$. The same lower-hull optimization and better-of-two Lagrangian argument used in
the budget case gives
\[
    \calP_T
    =
    O\!\left(\Lambda
        \parr*{d\sqrt T\log^2T
        +\Delta_{\hob}\sqrt{dT}\log T+H_\hob+\frac{T}{K}}
    \right),
\]
up to the separately controlled exploration contribution. For regret, weak
duality compares the algorithm to the convexified optimum at the current
\(\lambda_t\), and the mirror-descent update with comparator \(0\) controls the
dual term \(\sum_t\lambda_t g_t(b_t)\).

For violation, the same mirror-descent inequality is applied with comparator
\(\Lambda\). Rearranging the resulting bound converts cumulative negative
slack into \(\Viol_T^\ros\). The dual-aware exploitation rounds contribute the
same statistical terms as above, while value-exploration rounds add
\(O((d^{1/2}+\Delta_{\hob}^{1/2})T^{3/4})\) through the threshold defining
when the value radius is small enough for Lagrangian optimization. The grid
approximation contributes \(T/(K\Lambda)\) after converting the dual comparison
into a violation bound.
Again, Theorem \ref{thm:ros_regret_violation} is agnostic to the HoB feedback. Under the specific feedbacks we have considered and implemented, the end-to-end results are given below.

\begin{corollary}[RoS rates under HoB feedback]\label{cor:ros_feedback}
Under the conditions of Theorem~\ref{thm:ros_regret_violation}, the
full-information HoB instantiation with $K=\Omega(\sqrt{T})$ satisfies
\[
\begin{aligned}
    \Reg_T^{\ros,\full}
    &=
    \widetilde O\parr*{\Lambda d\sqrt T},
    &
    \Viol_T^{\ros,\full}
    &=
    \widetilde O\parr*{
        d\sqrt T+\sqrt d\,T^{\frac{3}{4}}
    } .
\end{aligned}
\]
The binary HoB instantiation with $K=\Omega(T^{\frac{1}{3}})$
satisfies
\[
\begin{aligned}
    \Reg_T^{\ros,\bin}
    &=
    \widetilde O\parr*{
        \Lambda\parr*{d\sqrt T+d^{\frac{2}{3}}T^{\frac{2}{3}}}
    },
    \\
    \Viol_T^{\ros,\bin}
    &=
    \widetilde O\parr*{
        d\sqrt T
        +d^{\frac{2}{3}}T^{\frac{2}{3}}
        +d^{\frac{1}{2}}T^{\frac{3}{4}}
        +d^{\frac{1}{12}}T^{\frac{5}{6}}
    } .
\end{aligned}
\]
In particular, for fixed \(d\), these regret rates are
\(\widetilde O(\sqrt T)\) and \(\widetilde O(T^{\frac{2}{3}})\), respectively.
\end{corollary}

\section{Numerical Experiments}
\label{sec:numerical_experiments}

We evaluate the constrained bidding algorithms on semi-synthetic experiments
built from the iPinYou real-time bidding dataset~\citep{zhang2014real}. We use
four Season~3 campaigns with identifiers 2259, 2261, 2821, and 2997. Contexts
are built from auction-side categorical fields, excluding click and price
fields, and mapped to normalized hashed representations of dimension \(64\) for
full information and \(8\) for binary feedback. Since iPinYou logs come from
second-price auctions, we use the logged bid price as a proxy for \(v_{t,1}\)
and generate latent uplift and \(\hob\) from controlled contextual linear
models:
\[
    v_t=\Pi_{[0,0.8]}(\theta_*^\top x_t),
    \qquad
    v_{t,0}=v_{t,1}-v_t,
    \qquad
    m_t=\Pi_{[0,1]}(\phi_*^\top x_t+\xi_t).
\]
\DIFadd{For every observation retained in the experiments, the normalized
logged bid-price proxy satisfies \(v_{t,1}\in[0.8,1]\). Because the generated
uplift satisfies \(v_t\in[0,0.8]\), the construction ensures
\(v_{t,0}=v_{t,1}-v_t\in[0,1]\). Thus both potential outcomes satisfy the
boundedness assumption without additional projection.}
All curves average over 10 seeds. We compare Dual-LTE with a single-branch
optimistic baseline (1UCB), a plug-in primal-dual baseline, and an
\(\varepsilon\)-greedy baseline. In the Budget experiments, Budget-PD
implements the budget primal-dual rule of \citet{wang2023learning}, whereas
RoS-PD implements the RoS primal-dual rule of \citet{li2025no}.
Both plug-in baselines replace the unknown \(v_t\) by the current estimate
\(\widehat v_t=\widehat\theta_t^\top x_t\). All baselines use the same value
and \(\hob\) estimators and bid grid across methods.

\subsection{Full-information Feedback}
\label{sec:ipinyou_full_information_experiments}

In full-information feedback, the learner observes the realized \(\hob\) after
each auction. For the Budget experiments, we set \(B=T/4\), use \(T=30000\),
and measure regret against the budget-feasible benchmark.

\begin{figure}[!t]
    \vspace{-0.3em}
    \centering
    \includegraphics[width=\textwidth]{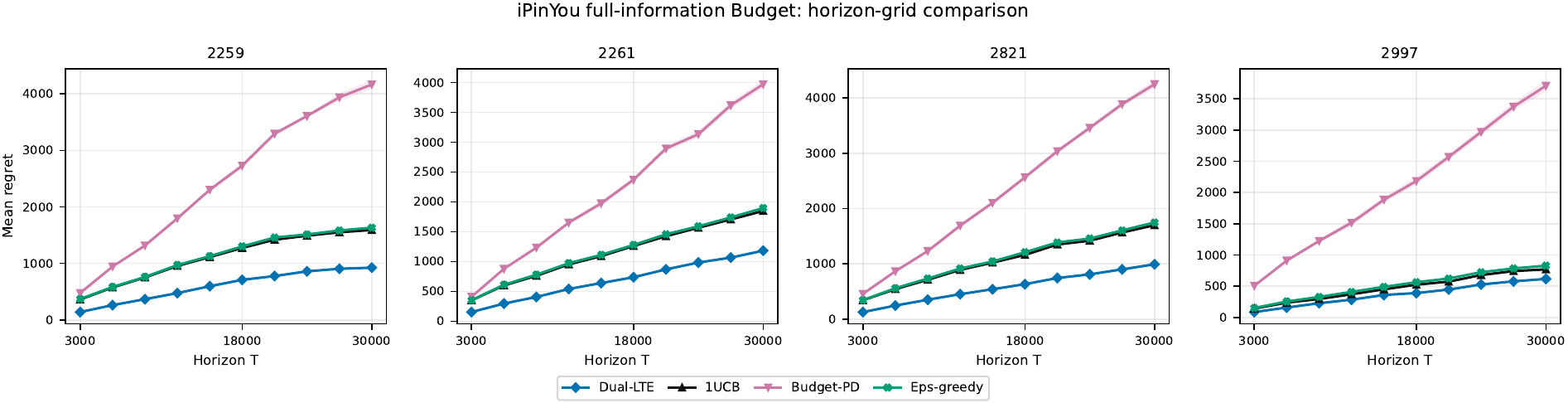}
    \vspace{-0.5em}
    \caption{Full-information Budget regret on four iPinYou campaigns. Curves
    show means over 10 seeds, with shaded bands showing \(\pm 1\) standard
    error.}
    \label{fig:budget_ipinyou_full}
\end{figure}

Figure~\ref{fig:budget_ipinyou_full} shows that Dual-LTE obtains the lowest
final regret on all four campaigns, reflecting the benefit of pairing value
learning with the budget shadow price.
Although the regret panels do not directly display budget usage, the recorded
spend ratios at \(T=30000\) show that Dual-LTE nearly depletes the budget on
all four campaigns, with spend ratios between \(0.989\) and \(0.992\). By
contrast, Budget-PD substantially under-spends, with spend ratios between
\(0.325\) and \(0.385\), which explains its larger regret. The improvement is
therefore not only due to value learning, but also to using the budget shadow
price when deciding whether and how to explore.

\smallskip
\noindent\textit{RoS constraint.}\quad
We next impose an RoS target \(\kappa=1\), so feasibility is governed by the
RoS slack defined in Section~\ref{sec:formulation}. Accordingly, regret is
measured relative to the RoS-feasible benchmark, and violation is measured as
the cumulative RoS constraint violation. The plug-in primal-dual baseline in
this setting is RoS-PD, which uses the algorithm of \citet{li2025no} with
\(v_t\) replaced by \(\widehat v_t\).

\begin{figure}[!htbp]
    \vspace{-0.3em}
    \centering
    \includegraphics[width=\textwidth]{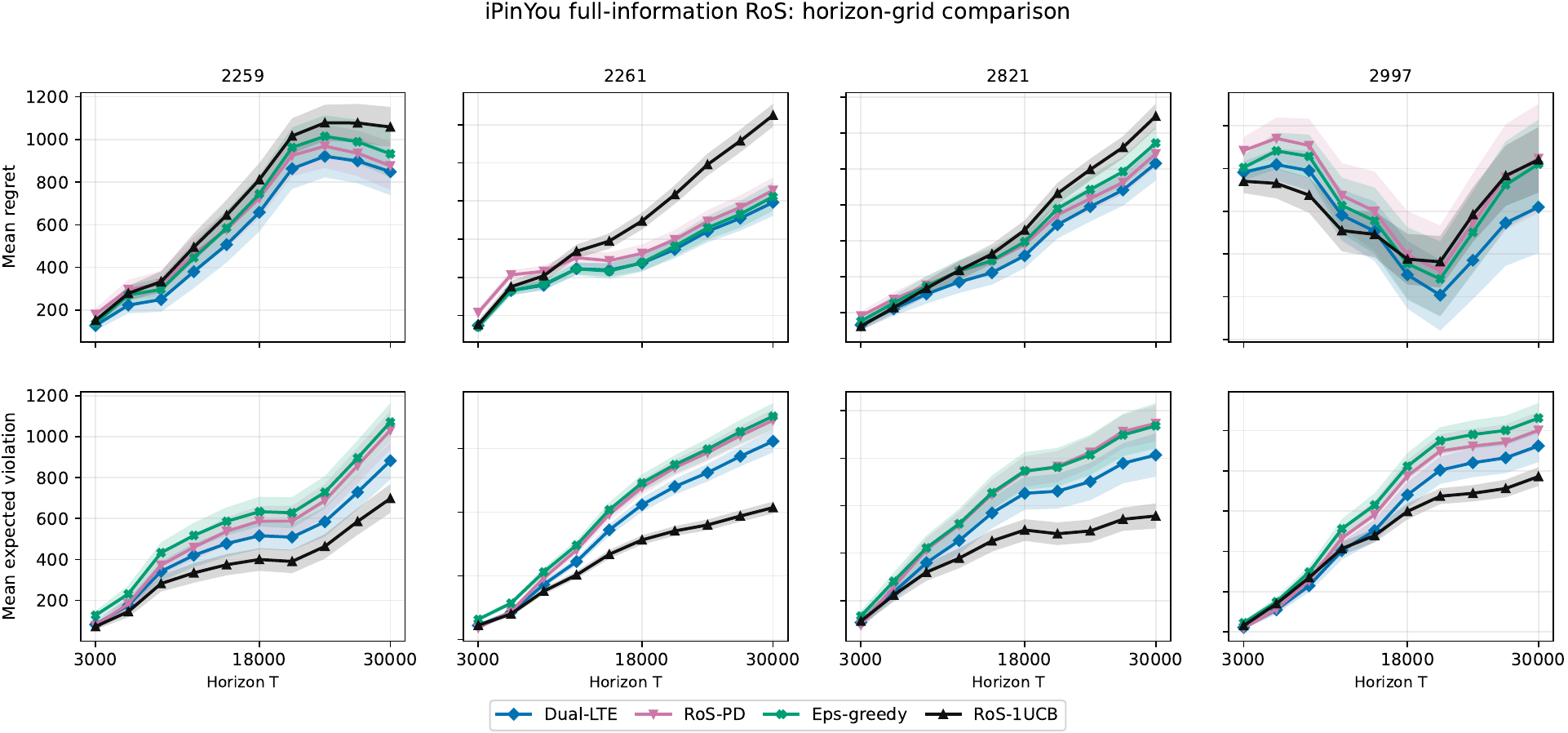}
    \vspace{-0.5em}
    \caption{Full-information RoS regret and expected violation on four iPinYou
    campaigns. Curves show means over 10 seeds, with shaded bands showing
    \(\pm 1\) standard error.}
    \label{fig:ros_ipinyou_full}
\end{figure}

Figure~\ref{fig:ros_ipinyou_full} shows that Dual-LTE has the lowest regret on
all four campaigns. The tuned 1UCB baseline is more conservative and has lower
violation in this instance, so the full-information RoS results illustrate a
regret--violation tradeoff rather than uniform dominance on both metrics.
This is consistent with the role of the RoS multiplier: a more conservative
policy can protect the constraint by forgoing value, while Dual-LTE uses the
dual-aware Lagrangian to trade off value learning and RoS slack.

\subsection{Binary Feedback}
\label{sec:ipinyou_binary_feedback_experiments}

We next repeat the experiments under binary feedback, where the learner observes
only win/loss outcomes and estimates the \(\hob\) distribution using
Section~\ref{sec:bin_hob_est}.

\smallskip
\noindent\textit{Budget constraint under binary feedback.}\quad
For the Budget experiments, we keep \(B=T/4\), use horizons
\(9000,18000,\ldots,90000\), and compare against the same Budget baselines,
now with the binary \(\hob\) estimator.

\begin{figure}[!htbp]
    \vspace{-0.3em}
    \centering
    \includegraphics[width=\textwidth]{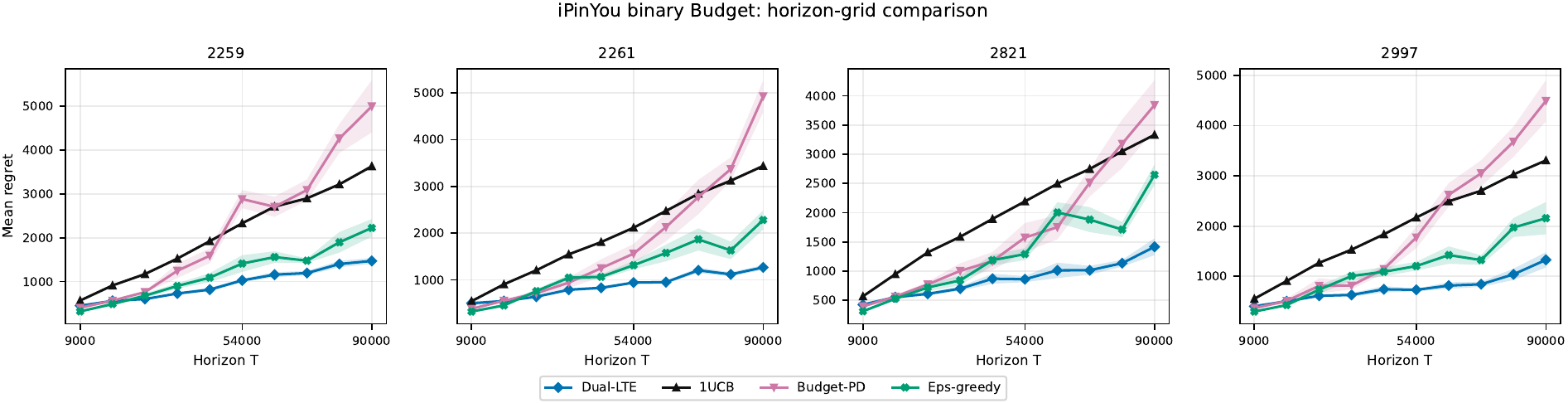}
    \vspace{-0.5em}
    \caption{Binary-feedback Budget regret on four iPinYou campaigns. Curves
    show means over 10 seeds, with shaded bands showing \(\pm 1\) standard
    error.}
    \label{fig:budget_ipinyou_binary}
\end{figure}

Figure~\ref{fig:budget_ipinyou_binary} shows that Dual-LTE has the lowest
regret on all four campaigns.
The binary estimator is less informative than full-information feedback, so
early exploration is more visible in the regret curves. Nevertheless, the same
constraint-aware exploration rule directs exploration toward bids that remain
aligned with the current budget state.

\smallskip
\noindent\textit{RoS constraint under binary feedback.}\quad
\DIFadd{For RoS, binary feedback is particularly demanding because inaccurate
HoB estimates early in the horizon can generate accumulated RoS debt. We
therefore evaluate Dual-LTE-SR, an experimental implementation heuristic that
augments Dual-LTE with HoB calibration and safe recovery. It initially uses
randomized exploration bids to initialize the value and HoB estimators. It then
continues HoB calibration using a conservative dual-aware bidding rule and
imposes a lower bound on the RoS dual variable to discourage bids that would
increase the accumulated RoS deficit. After calibration, it returns to the
Dual-LTE bid-selection rule while retaining the recovery safeguard. Dual-LTE-SR
is used only in the numerical experiments and is not part of
Algorithm~\ref{alg:dual-aware-lte} or covered by its theoretical guarantees.}
We compare against the same RoS baselines, including RoS-PD with the binary
\(\hob\) estimator and plug-in value estimate \(\widehat v_t\).

\begin{figure}[!htbp]
    \vspace{-0.3em}
    \centering
    \includegraphics[width=\textwidth]{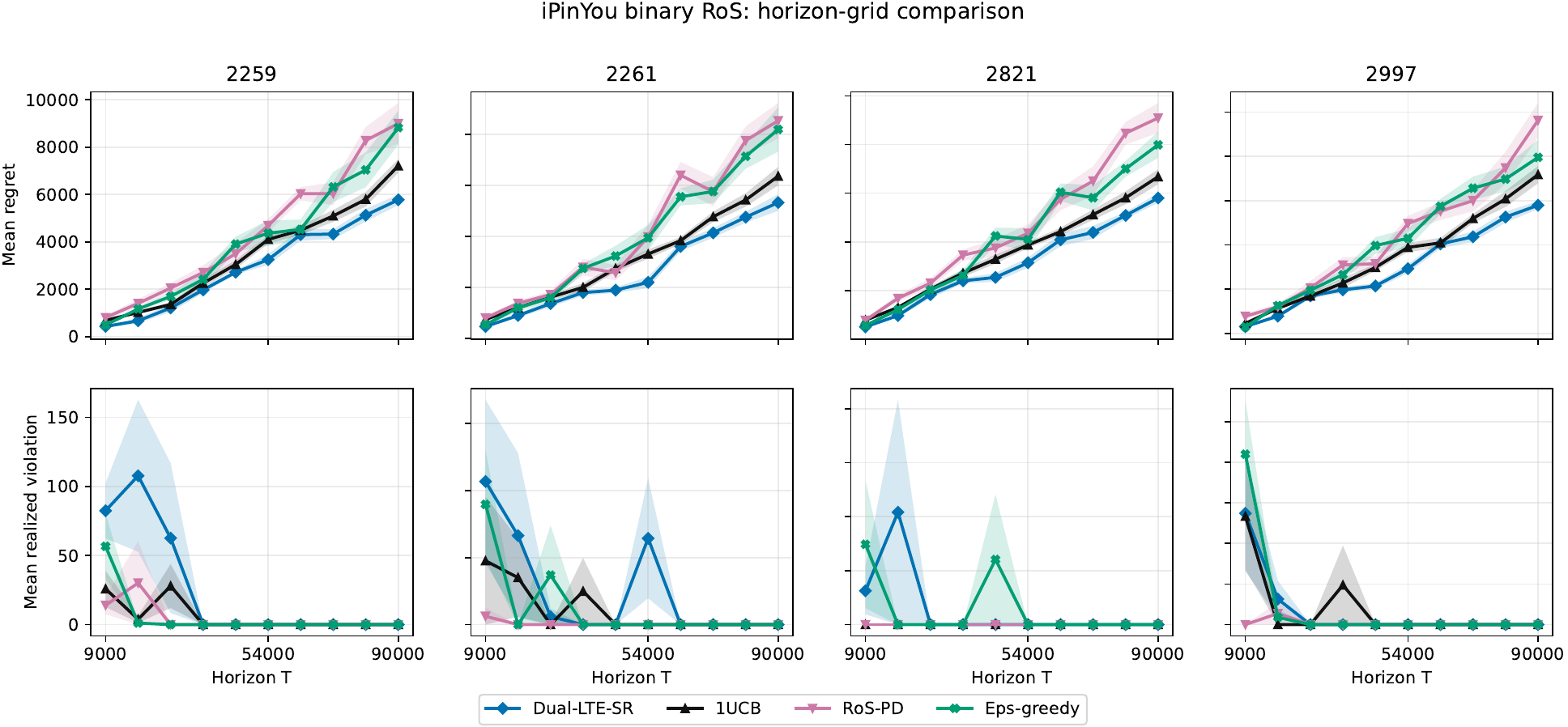}
    \vspace{-0.5em}
    \caption{Binary-feedback RoS regret and realized violation on four iPinYou
    campaigns. Curves show means over 10 seeds, with shaded bands showing
    \(\pm 1\) standard error. Dual-LTE-SR includes the safe-recovery calibration
    phase.}
    \label{fig:ros_ipinyou_binary}
\end{figure}

Figure~\ref{fig:ros_ipinyou_binary} shows that, as the horizon grows, all
methods eventually drive realized RoS violation to zero, while Dual-LTE-SR
achieves the lowest average regret.
The safe-recovery phase is useful because binary HoB feedback delays reliable
CDF calibration. During this early period, aggressive exploitation can create
RoS debt that is difficult to repay. Dual-LTE-SR uses the calibration phase to
stabilize the dual state before switching to the ordinary Lagrangian
comparison.

\section{Conclusion}
\label{sec:conclusion}

\DIFadd{This paper develops a dual-aware learning framework for automated
bidding in repeated first-price auctions under budget and RoS constraints. The
framework integrates treatment-effect learning with HoB estimation and
coordinates exploration with the current constraint pressure. Our analysis
shows how feedback quality affects achievable performance: the framework
attains \(\widetilde O(\sqrt T)\) regret under full-information HoB feedback and
\(\widetilde O(T^{2/3})\) regret under binary feedback, together with the
corresponding constraint guarantees. Semi-synthetic experiments using iPinYou
auction covariates show that the proposed approach achieves lower regret than
the considered baselines across the budget and RoS settings while maintaining
favorable constraint performance.}

\DIFadd{Operationally, our results suggest that auto-bidders should not
separate value learning from budget pacing or RoS control. Exploration should
instead respond jointly to estimation uncertainty and the current shadow price
of the advertiser's constraint. This coordination enables the bidder to gather
information when it is most valuable while limiting decisions that could
create budget inefficiency or accumulated RoS deficits.}

\bibliography{ref}

\ECSwitch
\ECHead{Electronic Companion}

\section{Auxiliary Estimation Results}
\label{app:estimation}

\subsection{Proof of Lemma \ref{lem:full_hob_est}}
\begin{proof}
 \DIFaddbegin \DIFadd{It suffices to consider \(t\ge2\), since the claim at \(t=1\) follows from
the unit range of a CDF. Define
}\[
    \DIFadd{e_\tau=(\hphi_\tau-\phi_*)^\top x_\tau.
}\]
\DIFadd{Because }\DIFaddend \(\hphi_\tau\) is computed from  \DIFaddbegin \DIFadd{observations before round \(\tau\),
the standard self-normalized inequality for ridge regression
gives}\DIFaddend , with probability at least
 \DIFaddbegin \DIFadd{\(1-\frac{1}{2}T^{-3}\),
}\begin{equation}\DIFadd{\label{eq:full_shift_confidence}
    |e_\tau|
    \le
    \beta_\hob\|x_\tau\|_{A_{\tau,\full}^{-1}}
    \qquad\text{simultaneously for every }\tau\in[T].
}\end{equation}\DIFaddend 
 \DIFaddbegin \DIFadd{Since \(\phi_*^\top x_\tau=\E[m_\tau\mid x_\tau]\in[0,1]\), and
\(A_{\tau,\full}\succeq\lambda_\phi I\), the event
\eqref{eq:full_shift_confidence} implies that the proxy residuals
\(\widehat\xi_\tau=m_\tau-\hphi_\tau^\top x_\tau\) and all evaluation
thresholds \(b-\hphi_t^\top x_t\) lie in
}\[
    \DIFadd{\mathcal U_T
    =\left[-1-\frac{\beta_\hob}{\sqrt{\lambda_\phi}},
       1+\frac{\beta_\hob}{\sqrt{\lambda_\phi}}\right].
}\]\DIFaddend 
 \DIFaddbegin \DIFadd{For fixed \(u\in\mathcal U_T\), predictability of \(e_\tau\) and
Assumption~\ref{ass:cond_ind} give
}\[
    \DIFadd{\E\!\left[\1[\widehat\xi_\tau\le u]
       \mid\mathcal H_{\tau-1},x_\tau\right]
    =G(u+e_\tau).
}\]\DIFaddend 
 \DIFaddbegin \DIFadd{Azuma--Hoeffding concentration, calibrated at each grid point and followed by
}\DIFaddend a union bound over  \DIFaddbegin \DIFadd{\(t\le T\) and a \(T^{-1/2}\)}\DIFaddend -grid of  \DIFaddbegin \DIFadd{\(\mathcal U_T\),
therefore gives, with probability at least \(1-\frac{1}{2}T^{-3}\),
}\begin{equation}\DIFadd{\label{eq:full_grid_concentration}
    \left|
    \frac{1}{t-1}\sum_{\tau<t}
       \left(\1[\widehat\xi_\tau\le u]-G(u+e_\tau)\right)
    \right|
    \le C\sqrt{\frac{\log T}{t-1}}
}\end{equation}
\DIFadd{simultaneously for all grid points and all \(t\ge2\).
}

\DIFadd{The bias induced by the predictable shifts is uniform in \(u\). Indeed,
Lipschitzness of \(G\), \eqref{eq:full_shift_confidence}, and the predictable
elliptical-potential bound give
}\begin{align}
\DIFadd{\left|
\frac{1}{t-1}\sum_{\tau<t}G(u+e_\tau)-G(u)
\right|
}&\DIFadd{\le
\frac{L\beta_\hob}{t-1}
\sum_{\tau<t}\|x_\tau\|_{A_{\tau,\full}^{-1}}
\notag}\\
&\DIFadd{\le
L\beta_\hob
\sqrt{\frac{2d\log(1+(t-1)/(\lambda_\phi d))}{t-1}}.
\label{eq:full_proxy_bias}
}\end{align}
\DIFadd{To extend the grid-point bound, fix arbitrary \(u\in\mathcal U_T\), and let
\(u^-\le u\le u^+\) be adjacent grid points. Monotonicity gives
}\DIFaddend \[
     \DIFaddbegin \hG\DIFadd{_t(u^-)}\DIFaddend \le \DIFaddbegin \hG\DIFadd{_t(u)}\DIFaddend \le \DIFaddbegin \hG\DIFadd{_t(u^+),
    \qquad
    G(u^-)}\le \DIFadd{G(u)}\le \DIFadd{G(u^+).
}\DIFaddend \]
 \DIFaddbegin \DIFadd{Since \(G\) is \(L\)-Lipschitz and \(u^+-u^-\le T^{-1/2}\), combining
\eqref{eq:full_grid_concentration} and \eqref{eq:full_proxy_bias} yields
}\begin{equation}\DIFadd{\label{eq:full_uniform_cdf}
    \sup_{u\in\mathcal U_T}|\hG_t(u)-G(u)|
    \le
    C\beta_\hob\sqrt{\frac{d\log T}{t-1}},
}\end{equation}

\DIFadd{Absorbing the interpolation and lower-order terms into the universal constant
\(C\) yields the stated bound.}

\DIFadd{Finally, for every \(b\in[0,1]\),
}\DIFaddend \begin{align*}
 \DIFaddbegin \DIFadd{|F_t(b)-}\hF\DIFadd{_t(b)|
}\DIFaddend & \le
 \DIFaddbegin \DIFadd{|G(b-\phi_*^\top }\DIFaddend x_t \DIFaddbegin \DIFadd{)-G(b-}\hphi\DIFadd{_t^\top x_t)|}\\
&\DIFadd{\quad}\DIFaddend +
 \DIFaddbegin \DIFadd{|G(b-}\hphi\DIFadd{_t^\top x_t)-}\hG\DIFadd{_t(b-}\hphi\DIFadd{_t^\top x_t)|}\DIFaddend \\
&\le
L\beta_\hob\|x_t\|_{A_{t,\full}^{-1}}
+ \DIFaddbegin \DIFadd{C}\DIFaddend \beta_\hob \DIFaddbegin \DIFadd{\sqrt{\frac{d\log T}{t-1}}}\DIFaddend .
\end{align*}
\DIFaddbegin \DIFadd{Since \((t-1)^{-1}\le2t^{-1}\) for \(t\ge2\), }\DIFaddend

\DIFadd{absorbing this factor into the constant implicit in \(\beta_\hob\) yields the stated bound.}
\end{proof}

\subsection{Proof of Lemma \ref{lem:bin_hob_est}}
\begin{proof}
\DIFadd{Let \(\mathcal G_{\mathrm{shift}}\) denote the self-normalized shift
event in \eqref{eq:binary_shift_event}. On this event, if}
\(\tau\notin\mathcal E_x\), then
\begin{equation}\label{eq:trusted_residual_error}
    \abs*{\widehat u_\tau-(b_\tau-\phi_*^\top x_\tau)}
    =\abs*{(\hphi_\tau-\phi_*)^\top x_\tau}
    \le a_T.
\end{equation}
The membership indicator
\(q_{\tau,j}=\1[\tau\notin\mathcal E_x,\widehat u_\tau\in I_j]\) is
\(\mathcal F_\tau^-\)-measurable, and
\(q_{\tau,j}\bigl(o_\tau-G(b_\tau-\phi_*^\top x_\tau)\bigr)\) is a bounded
martingale difference. \DIFadd{Define \(\mathcal G_{\mathrm{bin}}\) as the
event on which, simultaneously for all \(t,j\) with \(n_{t,j}>0\),
\begin{equation}\label{eq:adaptive_bin_concentration}
    \left|\overline g_{t,j}
    -\frac{1}{n_{t,j}}
       \sum_{\substack{\tau<t:\,\tau\notin\mathcal E_x,\\
                       \widehat u_\tau\in I_j}}
       G(b_\tau-\phi_*^\top x_\tau)\right|
    \le r_{t,j}.
\end{equation}
For each \(j\), enumerate its predictable}
selection times as \(\tau_{j,1}<\tau_{j,2}<\cdots\). Azuma--Hoeffding for each
prefix \(n\le T\), followed by a union bound over \((j,n)\), gives
\DIFadd{\(\Pr(\mathcal G_{\mathrm{bin}}^c)\le\frac{1}{2}T^{-3}\). Indeed,}
\(r_{t,j}=\sqrt{4\ell_T/(n+1)}\ge\sqrt{2\ell_T/n}\), and
\(\ell_T=\log(4JT^4)\) \DIFadd{gives the stated failure probability. This
argument allows adaptive bids and bin assignments because their membership
indicators are predictable.}

\DIFadd{On \(\mathcal G_{\mathrm{shift}}\cap\mathcal G_{\mathrm{bin}}\), fix
\(t,j\) with \(n_{t,j}>0\).} For every summand in bin \(j\),
\eqref{eq:trusted_residual_error}, the bin diameter, and Lipschitzness give
\[
    \abs*{G(b_\tau-\phi_*^\top x_\tau)-G(u_j)}
    \le L(a_T+\operatorname{diam}(I_j))\le2La_T.
\]
Combining this with \eqref{eq:adaptive_bin_concentration} yields
\begin{equation}\label{eq:bin_center_confidence}
    \abs*{\overline g_{t,j}-G(u_j)}\le r_{t,j}+2La_T.
\end{equation}
Consequently
\[
    G(u_j)-2r_{t,j}-4La_T
    \le\underline g_{t,j}\le G(u_j).
\]
The upper inequality also holds for an empty bin because its lower estimate is
zero. Since \(G\) is nondecreasing and the representatives are ordered, for
\(u\in I_j\),
\[
    G(u)-2r_{t,j}-5La_T
    \le \widehat g_{t,j}
    =\max_{k\le j}\underline g_{t,k}
    \le G(u)+La_T.
\]
Thus
\begin{equation}\label{eq:monotone_bin_confidence}
    \abs*{\hG_t(u)-G(u)}\le2r_{t,j}+5La_T.
\end{equation}
When \(n_{t,j}=0\), the right side before capping exceeds one for \(T\ge2\),
so the same statement follows from the unit range of a CDF.

Finally, for \(b\in\calB\), use \eqref{eq:monotone_bin_confidence} at
\(u=b-\hphi_t^\top x_t\) and the shift event to obtain
\[
\begin{aligned}
    \abs*{F_t(b)-\hF_t(b)}
    &\le L\abs*{(\hphi_t-\phi_*)^\top x_t}
       +2r_{t,j_t(b)}+5La_T\\
    &\le L\beta_\hob\|x_t\|_{A_{t,\bin}^{-1}}
       +2r_{t,j_t(b)}+5La_T.
\end{aligned}
\]
Capping this expression at one preserves validity and gives
\eqref{eq:finite_binary_hob_width}.
\DIFadd{The events \(\mathcal G_{\mathrm{shift}}\) and
\(\mathcal G_{\mathrm{bin}}\) each fail with probability at most
\(\frac{1}{2}T^{-3}\). A union bound therefore gives the stated probability
of at least \(1-T^{-3}\).}
\end{proof}

\subsection{Derivation of the HoB Estimation Widths}
\label{app:hob_width_bounds}

For full HoB feedback, after excluding the CDF-width exploration rounds,
the standard leverage and empirical-CDF arguments give
\(D_{\hob,T}^2=\widetilde O(d)\) and
\(\Delta_\hob=\widetilde O(\sqrt d)\).

For binary feedback, \(\mathcal E_\hob=\mathcal E_x\cup\mathcal E_\delta\).
Every \(t\in\mathcal I_{\rm WLS}\) occurs when the shift-uncertainty condition
\eqref{eq:bin_hob_width_test} is not satisfied, so
\(L\beta_\hob\|x_t\|_{A_{t,\bin}^{-1}}\le La_T\). Moreover, every such round
increments the trustworthy count of its selected bin. Consequently, if \(M_j\)
is the number of WLS-eligible visits to bin \(j\),
\begin{equation}\label{eq:binary_hob_harmonic_squares}
    \sum_{t\in\mathcal I_{\rm WLS}}
       \frac{1}{N_{t,j_t(b_t)}}
    \le
    \sum_{j=1}^J\sum_{n=1}^{M_j}\frac{1}{n}
    \le J\bigl(1+\log(T+1)\bigr).
\end{equation}
Using \eqref{eq:finite_binary_hob_width},
\((a+b+c)^2\le3(a^2+b^2+c^2)\), and
\(a_T^2=d^{1/3}T^{-2/3}\), we obtain
\begin{equation}\label{eq:binary_cumulative_width_squared}
\begin{aligned}
    D_{\hob,T}^2
    &\le
    78L^2T a_T^2
    +48\ell_T
       \sum_{t\in\mathcal I_{\rm WLS}}
       \frac{1}{N_{t,j_t(b_t)}}\\
    &\le
    78L^2d^{1/3}T^{1/3}
    +48\ell_TJ\bigl(1+\log(T+1)\bigr)\\
    &=\widetilde O\!\left(d+L^2d^{1/3}T^{1/3}\right).
\end{aligned}
\end{equation}
Therefore, one may take $\Delta_\hob
    =\widetilde O\!\left(\sqrt d+L d^{1/6}T^{1/6}\right)$.
Under the standing convention \(L=O(1)\), this becomes
\(\widetilde O(\sqrt d+d^{1/6}T^{1/6})\). A newly accessible empty bin is
assigned to \(\mathcal E_\delta\), accounted for separately, and excluded from
\(D_{\hob,T}\). Its query increases the count used on later rounds.

\subsection{Proof of Lemma \ref{lem:bias-variance}}
\begin{proof}
\textbf{(Bias)}
Denote the bias by $\zeta_t(b) = \E_t[\widetilde{e}_t(b)] - \theta_*^{\top}x_t$. 
By taking expectation with respect to the HoB $m_t$, we have
\begin{align*}
\E_{m_t}\parq*{\widetilde{e}_t(b)} - \parr*{v_{t,1}-v_{t,0}} &= \parr*{\frac{F_t(b)}{\widehat{F}_t(b)}-1}v_{t,1} - \parr*{\frac{1-F_t(b)}{1-\widehat{F}_t(b)} -1}v_{t,0}.
\end{align*}
By the bounded potential-outcome condition, Jensen's inequality, and the
triangle inequality give, for  \DIFaddbegin \DIFadd{every }\DIFaddend \(b\) \DIFaddbegin \DIFadd{satisfying \(0<\hF_t(b)<1\)}\DIFaddend ,
\begin{align*}
\abs*{\E\parq*{\widetilde{e}_t(b)} - \theta_*^{\top}x_t} &= \abs*{\E_{v_{t,1}, v_{t,0}}\parq*{\E_{m_t}\parq*{\widetilde{e}_t(b)} - (v_{t,1}-v_{t,0})}}\\
&\leq \E_{v_{t,1}, v_{t,0}}\abs*{\E_{m_t}\parq*{\widetilde{e}_t(b)} - (v_{t,1}-v_{t,0})}\\
&\leq \abs*{\frac{F_t(b)}{\widehat{F}_t(b)}-1} + \abs*{\frac{1-F_t(b)}{1-\widehat{F}_t(b)} -1}\le 2\delta_t(b)\sigma_t(b)^2.
\end{align*}

\textbf{(Variance)} 
We bound the variance of \(\widetilde{e}_t(b)\) by its second moment. The
bounded potential-outcome condition gives
\begin{align*}
\Var(\widetilde{e}_t(b)) \le \E[\widetilde{e}_t(b)^2] &\overset{\text{(a)}}{\le} \frac{2\E[\1[b\ge m_t]]}{\widehat{F}_t(b)^2} + \frac{2\E[\1[b<m_t]]}{(1-\widehat{F}_t(b))^2}\\
&\le \frac{2}{\widehat{F}_t(b)^2} + \frac{2}{(1-\widehat{F}_t(b))^2}\le 4\sigma_t(b)^4
\end{align*}
where (a) follows from the AM-GM inequality $(a+b)^2 \le 2a^2 + 2b^2$.
\end{proof}

\subsection{Proof of Lemma \ref{lem:wls}}

\begin{proof}
Fix \(t\) \DIFaddbegin \DIFadd{. Ordinary rounds with
\(\hF_\tau(b_\tau)\in\{0,1\}\) have zero WLS weight and therefore contribute
nothing to \(A_t\) or \(y_t\). Let
\(\mathcal I_t\) denote the remaining rounds in
\(\mathcal I_{{\rm WLS},t}\)}\DIFaddend . For each
\(\tau\in\mathcal I_t\), define
\(\varepsilon_\tau=\te_\tau-\E_\tau[\te_\tau]\) and
\(\zeta_\tau=\E_\tau[\te_\tau]-\theta_*^\top x_\tau\). Then
\(\te_\tau=\theta_*^\top x_\tau+\varepsilon_\tau+\zeta_\tau\) and
\(\E_\tau[\varepsilon_\tau]=0\).
On endpoint value-exploration rounds, the estimator in
Section~\ref{sec:value_est} is unbiased, so \(\zeta_\tau=0\),
\(\sigma_\tau=1\), and \(\varepsilon_\tau\) is bounded. On every other
WLS-eligible round, Lemma~\ref{lem:bias-variance} gives
\[
    \abs*{\zeta_\tau}\le2\delta_\tau(b_\tau)\sigma_\tau^2,
    \qquad
    \Var_\tau(\te_\tau)\le4\sigma_\tau^4.
\]
No observation from \(\mathcal E_\hob\) appears in these definitions.

The ridge normal equations imply
\[
\htheta_t-\theta_*
=A_t^{-1}\!\left(
    \sum_{\tau\in\mathcal I_t}\sigma_\tau^{-4}x_\tau\varepsilon_\tau
   +\sum_{\tau\in\mathcal I_t}\sigma_\tau^{-4}x_\tau\zeta_\tau
   -\theta_*\right).
\]
For an ordinary IPW round,
\[
    \sigma_\tau^{-2}\te_\tau
    =(1-\hF_\tau(b_\tau))\1[b_\tau\ge m_\tau]v_{\tau,1}
    -\hF_\tau(b_\tau)\1[b_\tau<m_\tau]v_{\tau,0}
\]
lies in \([-1,1]\). The endpoint estimator is bounded by an absolute constant
as well. Therefore \(\sigma_\tau^{-2}\varepsilon_\tau\) is conditionally
sub-Gaussian with an absolute constant parameter. The self-normalized
martingale inequality applied only to
\(\{\sigma_\tau^{-2}x_\tau:\tau\in\mathcal I_t\}\) gives, simultaneously in
\(t\),
\[
    \left\|\sum_{\tau\in\mathcal I_t}
       \sigma_\tau^{-4}x_\tau\varepsilon_\tau\right\|_{A_t^{-1}}
    \le C_1\sqrt{d\log(T/\delta)}.
\]
For the predictable bias, the matrix inequality
\(X^\top(I+XX^\top)^{-1}X\preceq I\) gives
\[
\begin{aligned}
    \left\|\sum_{\tau\in\mathcal I_t}
       \sigma_\tau^{-4}x_\tau\zeta_\tau\right\|_{A_t^{-1}}
    &\le
    \sqrt{\sum_{\tau\in\mathcal I_t}
       (\sigma_\tau^{-2}\zeta_\tau)^2}\\
    &\le C_2\sqrt{\sum_{\tau\in\mathcal I_t}
       \delta_\tau(b_\tau)^2}.
\end{aligned}
\]
Finally \(A_t\succeq I\) and \(\|\theta_*\|_2\le1\). Combining the three
terms proves the stated bound. In binary mode,
\eqref{eq:binary_cumulative_width_squared} bounds the last display before this
lemma is invoked, so the WLS argument does not feed back into either HoB
exploration count.
\end{proof}

\section{Auxiliary Optimization and Online-Learning Results}
\label{app:online-learning}

\subsection{Proof of Lemma \ref{lem:lambda_bound}}
\begin{proof}
For any \(\lambda\ge0\),
\(\overline D(\lambda)\ge
\E_{x_t}[r_t(\pi^{\mathrm{sl}}(x_t))
+\lambda g_t(\pi^{\mathrm{sl}}(x_t))]
\ge\lambda\delta_S\).
Moreover,
\(\overline D(0)=\E_{x_t}[\max_{z\in\calC_t}r_t(z)]\le1\).
If a minimizer satisfied \(\lambda_*>1/\delta_S\), then
\(\overline D(\lambda_*)>1\ge\overline D(0)\), contradicting its optimality.
\end{proof}

\subsection{Proof of Lemma \ref{lem:branching_control}}

\smallskip
\noindent\textit{Proof structure.}\quad
The proof of Lemma~\ref{lem:branching_control} is split into three auxiliary
claims, one for each branch of Algorithm~\ref{alg:ucb_branching}. The final
paragraph of this subsection then assembles these claims to prove the lemma.

\begin{lemma}\label{lem:branching_small_est_err}
Let $c=\frac{1}{20L}$ and $[b_{\mathrm{left}},b_{\mathrm{right}}]$ be the interval found in Algorithm \ref{alg:ucb_selection}. When $c^2\rho_t \ge R_t + \Delta_t(\widehat{v}_t + R_t+\rho_t)$, it holds that
\begin{enumerate}
    \item  \DIFaddbegin \DIFadd{$\argmax_{b^*\in[0,1]}F_t(b^*)(v_t-\rho_t b^*)
    \subseteq[b_{\mathrm{left}},b_{\mathrm{right}}]$}\DIFaddend .

    \item
    $\widehat{F}_t(b_{\mathrm{right}}) - \widehat{F}_t(b_{\mathrm{left}}) \le 1 - 2c$.
\end{enumerate}
where $R_t\ge \abs*{\widehat{v}_t-v_t}$  \DIFaddbegin \DIFadd{and
}\DIFaddend $\Delta_t = \max_{b\in\calB}\delta_t(b)$ \DIFaddbegin \DIFadd{.
}\DIFaddend \end{lemma}
\begin{proof}
Write $\gamma_t
    =
    R_t+\Delta_t(\widehat v_t+R_t+\rho_t)$,
so the lemma assumption gives \(\gamma_t\le c^2\rho_t\). Recall that
\[
    b_{\pm}
    \in
    \argmax_b\widehat F_t(b)
    \parr*{\widehat v_t\pm 2\rho_tc-\rho_tb},
\]
and
\[
    U
    =
    \bra*{
        b:\widehat F_t(b_-)-c
        \le \widehat F_t(b)
        \le \widehat F_t(b_+)+c
    }.
\]
Let \(b_t^*\in\argmax_b F_t(b)(v_t-\rho_tb)\). We first show that \(b_t^*\)
is not below the returned candidate interval. Using the value and HoB
confidence bounds, the optimality of \(b_-\) for the empirical score with the
downward value perturbation, and \(b_t^*\)-optimality for the population score,
we obtain
\begin{align*}
    F_t(b_t^*)(v_t-\rho_tb_t^*)
    &\le
    F_t(b_-)(v_t-\rho_tb_-)
    +2\gamma_t
    +2\rho_tc\parr*{\widehat F_t(b_t^*)-\widehat F_t(b_-)} .
\end{align*}
Because \(b_t^*\) maximizes the population score,
\[
    F_t(b_t^*)(v_t-\rho_tb_t^*)
    \ge
    F_t(b_-)(v_t-\rho_tb_-).
\]
Comparing this lower bound with the preceding upper bound gives
\[
    \widehat F_t(b_t^*)-\widehat F_t(b_-)
    \ge
    -\frac{\gamma_t}{c\rho_t}
    \ge -c .
\]
A symmetric argument using the upward value perturbation and the maximizer
\(b_+\) gives
\(\widehat F_t(b_t^*)\le\widehat F_t(b_+)+c\). Hence
\(b_t^*\in[b_{\mathrm{left}},b_{\mathrm{right}}]\).

For the width of the returned interval, the definition of \(U\) gives
\[
    \widehat{F}_t(b_{\mathrm{right}})
    - \widehat{F}_t(b_{\mathrm{left}})
    \le
    \widehat F_t(b_+)-\widehat F_t(b_-)+2c .
\]
Under the same confidence condition and the grid mesh used in
Algorithm~\ref{alg:ucb_selection}, Lemma~\ref{lem:approx_CDF_UCB_Lip} gives
\(\widehat F_t(b_+)-\widehat F_t(b_-)\le 1-\frac{1}{5L}\). Since
\(c=1/(20L)\),
\[
    \widehat{F}_t(b_{\mathrm{right}})
    - \widehat{F}_t(b_{\mathrm{left}})
    \le
    1-\frac{1}{5L}+2c
    = 1-\frac{1}{10L}
    = 1-2c
\]
as desired.
\end{proof}

\begin{lemma}\label{lem:one_side_cdf}
Fix any $\zeta\in(0,\frac{1}{4}).$ If 
$\widehat{v}_t = \widehat{\theta}_t^\top x_t \ge R_t + \frac{\rho_t}{\zeta},$
then $1-\widehat{F}_t(b_t^*) \le \zeta + \Delta_t$
where $\Delta_t = \max_{b\in\calB}\delta_t(b)$.
\end{lemma}
\begin{proof}
Since \(\abs*{v_t-\widehat v_t}\le R_t\), the high-value condition
\(\widehat v_t\ge R_t+\rho_t/\zeta\) implies \(v_t\ge\rho_t/\zeta\).
If \(b_t^*\) maximizes
\(F_t(b)(v_t-\rho_tb)\), then comparing its objective value with that of bid
\(1\), for which \(F_t(1)=1\), gives
\[
    F_t(b_t^*)v_t-\rho_tb_t^*F_t(b_t^*)
    \ge v_t-\rho_t .
\]
Since \(b_t^*F_t(b_t^*)\le1\), this implies
\[
    F_t(b_t^*)\ge 1-\frac{\rho_t}{v_t}\ge 1-\zeta .
\]
The HoB confidence bound then yields
\[
    1-\widehat F_t(b_t^*)
    \le
    1-F_t(b_t^*)+\Delta_t
    \le \zeta+\Delta_t .
\]
\end{proof}

By Lemma \ref{lem:one_side_cdf}, we can restrict to the bid subspace
\(\calB_t \coloneqq \bra{b: \widehat{F}_t(b) \ge 1-\zeta-\Delta_t}\).
In particular, when \(\zeta-\Delta_t \le \frac{1}{2}\), for any
\(b\in\calB_t\), we have
\(1-\widehat{F}_t(b) \le 2\widehat{F}_t(b)(1-\widehat{F}_t(b))\).
Next, we consider the case when the condition in Lemma \ref{lem:one_side_cdf} fails.

\begin{lemma}\label{lem:explore_branching_bound}
Fix any $\zeta,c\in(0,\frac{1}{4})$. Suppose $\widehat{v}_t < R_t + \frac{\rho_t}{\zeta}$ and $c^2\rho_t < R_t + \Delta_t(\widehat{v}_t + R_t + \rho_t)$. If $\Delta_t \le \frac{c^2\zeta}{3}$,
then \DIFaddbegin \DIFadd{\(\rho_t=O(R_t)\) and \(v_t=O(R_t)\), and
consequently any bid \(b_t\in\calB\) incurs at most \(O(R_t)\) one-round
Lagrangian loss}\DIFaddend .
\end{lemma}
\begin{proof}
 \DIFaddbegin \DIFadd{If \(\rho_t=0\), failure of the high-value condition gives
\(v_t\le2R_t\), and the result follows. Suppose henceforth that
\(\rho_t>0\). }\DIFaddend Since the high-value condition does not hold, the value
confidence bound gives
\[
    v_t
    \le
    \widehat v_t+R_t
    \le
    2R_t+\frac{\rho_t}{\zeta}.
\]
We claim \(R_t\ge c^2\rho_t/4\). Otherwise,
\(R_t<c^2\rho_t/4\), and the previous display implies
\[
    \widehat v_t+R_t+\rho_t
    <
    \rho_t\parr*{\frac{c^2}{2}+1+\frac{1}{\zeta}} .
\]
Combining this with \(\Delta_t\le c^2\zeta/3\) and
\(c,\zeta\in(0,1/4)\) gives
\[
    R_t+\Delta_t(\widehat v_t+R_t+\rho_t)
    < c^2\rho_t,
\]
contradicting the failed small-estimation-error condition. Hence
\(\rho_t\le 4R_t/c^2\), and therefore
\[
    v_t
    \le
    2R_t+\frac{\rho_t}{\zeta}
    \le
    R_t\parr*{2+\frac{4}{c^2\zeta}}
    =
    O(R_t).
\]
Since \(\rho_t=O(R_t)\)  and \(b,F_t(b)\in[0,1]\), replacing the
Lagrangian-maximizing bid by any exploration bid changes the one-round
Lagrangian value by at most \(O(v_t+\rho_t)=O(R_t)\).
\end{proof}

\smallskip
\noindent\textit{Completion of the proof of Lemma~\ref{lem:branching_control}.}\quad
Recall that Lemma~\ref{lem:branching_control} localizes the optimizer with
respect to the population convex hull \(\calC_t^\star\). Because the objective
\((q,c)\mapsto v_tq-\rho_tc\) is linear, it suffices to consider the
deterministic grid points. Any optimal hull action is a convex combination of
points \((F_t(b),bF_t(b))\) whose bids maximize
\(v_tF_t(b)-\rho_tbF_t(b)\).
Now by Lemmas~\ref{lem:branching_small_est_err} and \ref{lem:one_side_cdf},
when Algorithm~\ref{alg:ucb_branching} returns \(u_t=0\), the bids \(b\) of
the deterministic maximizers lie in the returned bid set \(\calB_t\).
Consequently, any maximizing hull action \(z\) lies in the convex hull of
\(\calB_t\). It remains to prove the two coefficient bounds in the
``moreover'' statement, which relate the selected value-error coefficient to
the Bernoulli variance \(\hF_t(b)(1-\hF_t(b))\). These bounds follow from the
same branch certificates: in the high-value branch, Lemma~\ref{lem:one_side_cdf}
and the definition of \(\calB_t\) give the bound for \(i_t=1\). In the
small-estimation-error branch, Lemma~\ref{lem:branching_small_est_err} and the
index returned by Algorithm~\ref{alg:ucb_selection} give the bound for the
selected \(i_t\). If \(u_t=1\),
\DIFadd{the choices \(c=1/(20L)\) and \(\zeta=1/8\), together with failure of
the CDF-width exploration condition, imply}
\[
    \DIFadd{\Delta_t
    \le \frac{1}{9600L^2}
    =\frac{c^2\zeta}{3}.
}
\]
\DIFadd{Hence the CDF-error condition in
Lemma~}\ref{lem:explore_branching_bound}\DIFadd{ holds whenever
Algorithm~}\ref{alg:ucb_branching}\DIFadd{ is invoked. Therefore,}
Lemma~\ref{lem:explore_branching_bound}  \DIFaddbegin \DIFadd{gives an \(O(R_t)\) bound on the
one-round Lagrangian loss.
}\DIFaddend


\section{Main Theorem Proofs and Additional Technical Lemmas}
\label{app:main-proofs}

\subsection{Bounding Regret during the Exploration Rounds}

Algorithm~\ref{alg:dual-aware-lte} has four disjoint exploration cases. Under binary HoB
feedback, \(\mathcal E_x\) contains shift-regression rounds and
\(\mathcal E_\delta\) contains selected-bin width queries. Let
\(\mathcal E_{\rm wid}\) contain rounds on which the value-radius condition is
satisfied and \(\mathcal E_{\rm val}\) those selected by the branch-certification
flag.

\begin{lemma}[Bounded Exploration Regret]\label{lem:bounded_explore}
Suppose the HoB and WLS confidence events hold. Then
\[
    |\mathcal E_{\rm wid}|
    =O\!\left(d\sqrt T\log^2T
      +\Delta_\hob\sqrt T\log T\right).
\]
For full HoB feedback, \(|\mathcal E_\hob|=O(d\log^2T)\). For binary HoB
feedback,
\begin{align}
    |\mathcal E_x|
    &\le
    2\beta_\hob^2d^{2/3}T^{2/3}
      \log\!\left(1+\frac{T}{\lambda_\phi d}\right)
    =O\!\left(d^{2/3}T^{2/3}\log^2T\right),
    \label{eq:shift_exploration_count}\\
    |\mathcal E_\delta|
    &\le
    256c_\hob^{-2}J\ell_T
    +\left(\frac{12L}{c_\hob}\right)^3d^{1/2},
    \qquad c_\hob= \DIFaddbegin \DIFadd{\frac{1}{9600L^2}}\DIFaddend ,
    \label{eq:width_exploration_count}
\end{align}
and hence \(|\mathcal E_\delta|=\widetilde O(J)\) asymptotically for fixed
\((d,L)\), while
\[
    |\mathcal E_\hob|
    =|\mathcal E_x|+|\mathcal E_\delta|
    =\widetilde O(d^{2/3}T^{2/3}).
\]
Finally,
\[
    \sum_{t\in\mathcal E_{\rm val}}
       \bigl(\ell_t^0(z_t^*)-\ell_t^0(z_t)\bigr)
    =O\!\left(d\sqrt T\log T
      +\Delta_\hob\sqrt{dT\log T}\right).
\]
Each round in \(\mathcal E_\hob\) contributes at most \(O(1)\) one-step loss in
the budget analysis and \(O(\Lambda)\) in the RoS analysis.
\end{lemma}

\begin{proof}
On \(\mathcal E_{\rm wid}\), neither HoB-exploration condition is satisfied.
the balanced endpoint estimator has \(\sigma_t=1\), and the round is
WLS-eligible. Hence
\[
\begin{aligned}
 &(d^{1/2}+\Delta_\hob^{1/2})T^{-1/4}|\mathcal E_{\rm wid}|<\sum_{t\in\mathcal E_{\rm wid}}R_t
 \le\beta_T\sqrt{2d|\mathcal E_{\rm wid}|\log T}.
\end{aligned}
\]
Since \(\beta_T=O(\sqrt{d\log T}+\Delta_\hob)\), rearrangement proves the
first claim.

Under full feedback, the full-information states are updated on every round.
Summing the failed width test and applying the elliptical-potential and
harmonic-time bounds gives \(|\mathcal E_\hob|=O(d\log^2T)\).

For binary shift exploration, let \(M_x=|\mathcal E_x|\). The shift test and
the fact that \(A_{t,\bin}\) is updated on every and only every round in
\(\mathcal E_x\) give
\begin{equation}\label{eq:shift_count_core}
\begin{aligned}
    a_TM_x
    &<\beta_\hob\sum_{t\in\mathcal E_x}
       \|x_t\|_{A_{t,\bin}^{-1}}\le\beta_\hob
       \sqrt{2dM_x
       \log\!\left(1+\frac{M_x}{\lambda_\phi d}\right)}.
\end{aligned}
\end{equation}
The factor \(\beta_\hob\) is retained. Squaring, using
\(a_T^{-2}=d^{-1/3}T^{2/3}\), and bounding \(M_x\le T\) proves
\eqref{eq:shift_exploration_count}. No elliptical-potential inequality is
applied to \(\mathcal E_\delta\), whose contexts do not update
\(A_{t,\bin}\).

It remains to count selected-bin width queries. Define
\(M_\delta=|\mathcal E_\delta|\), and let \(M_j\) be the number of times a
width-driven round selects bin \(j\). On such a round the shift test is false,
and the selected bid and accessible bin are
\[
    b_t\in\argmax_{b\in\calB}\delta_t(b),
    \qquad j_t=j_t(b_t).
\]
Thus \(L\beta_\hob\|x_t\|_{A_{t,\bin}^{-1}}+5La_T\le6La_T\). Since
\(\delta_t(b_t)>c_\hob\), \eqref{eq:finite_binary_hob_width} implies
\begin{equation}\label{eq:selected_width_gate_sum}
    (c_\hob-6La_T)M_\delta
    <4\sqrt{\ell_T}
      \sum_{t\in\mathcal E_\delta}
      \frac{1}{\sqrt{N_{t,j_t}}}.
\end{equation}
Every selected query increments its own trustworthy count. Before the
\(n\)-th width selection of bin \(j\), \(N_{t,j}\ge n\), even if no ordinary
round ever visits that bin. Therefore
\begin{equation}\label{eq:selected_bin_harmonic_count}
\begin{aligned}
    \sum_{t\in\mathcal E_\delta}
       \frac{1}{\sqrt{N_{t,j_t}}}
    &\le\sum_{j=1}^J\sum_{n=1}^{M_j}\frac{1}{\sqrt n}
    \le2\sum_{j=1}^J\sqrt{M_j}\le2\sqrt{J\sum_{j=1}^JM_j}
    =2\sqrt{JM_\delta}.
\end{aligned}
\end{equation}
This uses the bin actually selected from the currently accessible grid bids.
there is no maximum over all abstract bins.

If \(6La_T\le c_\hob/2\), substituting
\eqref{eq:selected_bin_harmonic_count} into
\eqref{eq:selected_width_gate_sum} and squaring gives
\(M_\delta\le256c_\hob^{-2}J\ell_T\). Otherwise
\(a_T>c_\hob/(12L)\), which implies
\(T<(12L/c_\hob)^3d^{1/2}\), and the trivial bound
\(M_\delta\le T\) gives the second term in
\eqref{eq:width_exploration_count}. No overlap or context-density condition is
used: inaccessible bins are irrelevant, and a bin that becomes accessible
with a large width is queried and increments its own count.

Finally, on \(\mathcal E_{\rm val}\), the branching lemmas give one-step
Lagrangian loss \(O(R_t)\). These balanced endpoint rounds are WLS-eligible and
have \(\sigma_t=1\), so the WLS elliptical-potential bound yields the final
claim.
\end{proof}

\subsection{Regret under Budget Constraint}
\label{app:budget_proofs}
\begin{proof}[Proof of Theorem \ref{thm:budget_regret}]
Let \(\mathcal G\) be the intersection of the HoB and WLS confidence events.
With \(\delta=T^{-3}\) in Lemma~\ref{lem:wls}, a union bound gives
\(\Pr(\mathcal G^c)\le2T^{-3}\). On \(\mathcal G\), first consider rounds outside
\(\mathcal E_\hob\cup\mathcal E_{\rm wid}\cup\mathcal E_{\rm val}\).
The contribution of the excluded rounds is bounded separately by
Lemma~\ref{lem:bounded_explore}.

Write $Z=\sqrt{\frac{LT}{B}}$ and recall the budget population Lagrangian
\[
    \ell_t^{\bgt}(z;\rho_t)=\ell_t^{0}(z;\rho_t)=v_tq_t(z)-\rho_t c_t(z),
    \qquad
    \rho_t=1+Z\mu_t,
\]
where $c_t(z)$ denotes the hull action payment as in \eqref{eq:def_pop_allocation_and_payment}.
Let \(z_t^*\sim\pi_{\bgt}^*(x_t)\) denote  hull action induced by the stationary budget-feasible
benchmark policy. Since it is stationary and budget-feasible, $\E_{x_t,m_t,z_t^*}[c(z_t^*)] \le \frac{B}{T}$.
Let \(\tau\in\{0,\ldots,T\}\) be the last round in which Algorithm \ref{alg:dual-aware-lte}
submits a bid before the budget depletes, with \(\tau=T\) if the budget is never exhausted. For $t\le \tau$,
\begin{align*}
\E[r_t(b_t^*) - r_t(b_t)] &= \E[\ell_t^0(z_t^*;\rho_t) - \ell_t^0(z_t;\rho_t)] + \E[Z\mu_t\parr*{c_t(z_t^*) - c_t(b_t)}]\\
&\le \E[\ell_t^0(z_t^*;\rho_t) - \ell_t^0(z_t;\rho_t)] + \E\parq*{Z\mu_t\parr*{\frac{B}{T} - c_t(b_t)}}.
\end{align*}
Since the randomized bid $b_t$
implements the selected hull action \(z_t\) in expectation, the regret
decomposes as
\begin{align}\label{eq:bgt_reg_decompose}
\Reg_T^{\bgt}
&\le
\underbrace{
\E\parq*{\sum_{t=1}^{\tau}
    \ell_t^{0}(z_t^*;\rho_t)-\ell_t^{0}(z_t;\rho_t)}
}_{(\spadesuit)}
+
\underbrace{
Z\E\parq*{\sum_{t=1}^{\tau}
    \mu_t\parr*{\frac{B}{T}-\bc_t(b_t)}}
}_{(\clubsuit)}
+\E\parq*{\sum_{t=\tau+1}^{T} r_t^{\bgt}(b_t^*)}.
\end{align}
Here \(\bc_t(b)=bF_t(b)\) denotes the expected payment. The first term is a budget analogue of the
generic Lagrangian-learning term. Following almost verbatim the proof of
Lemma~\ref{lem:bound_P_and_G}, with the budget shadow price
\(\rho_t=1+Z\mu_t=O(1)\) under Assumption~\ref{ass:linear_budget}, we have
\[
(\spadesuit)
=
O\parr*{
    d\sqrt T\log^2T
    +\Delta_{\hob}\sqrt{dT}\log T
    +H_\hob
    +\frac{T}{K}
}.
\]

It remains to control the pacing term and the regret after stopping. Since \(\mu_t\) is $\mathcal H_t$-measurable and the realized
cost \(c_t\) satisfies \(\E_t[c_t]=\bc_t(b_t)\),
\[
(\clubsuit)
\le
Z\E\parq*{\sum_{t=1}^{\tau}\mu_t\parr*{\frac{B}{T}-c_t}} .
\]
The update in \eqref{eq:budget_module} is mirror descent with feedback
\(\frac{B}{T}-c_t\). Lemma~\ref{lem:proj_mirror_des} therefore implies that, for any comparator
\(\mu\in[0,1]\),
\[
    \sum_{t=1}^{\tau}(\mu_t-\mu)\parr*{\frac{B}{T}-c_t}
    \le
    O(\sqrt B),
\]
where we use \(\eta=\Theta(T^{-1/2})\). Assumption~\ref{ass:linear_budget}
implies \(B=\Theta(T)\), while the budget-stopping rule gives
\(\sum_{t\le\tau}c_t\le B\). Finally, \(c_t\in[0,1]\) by construction. If
\(\tau<T\), then the stopping rule and the bound \(b_t\le1\) imply
\(\sum_{t=1}^{\tau}c_t>B-1\). Taking comparator \(\mu=1\) gives
\[
(\clubsuit)
\le
-(T-\tau)\sqrt{\frac{LB}{T}}+O(\sqrt T).
\]
If \(\tau=T\), taking comparator \(\mu=0\) gives the same bound with the
negative term equal to zero.

Finally, Assumption~\ref{ass:bound_density} gives \(F_t(b)\le Lb\). Hence the
per-round benchmark reward after stopping is bounded by
\[
\E[r_t^{\bgt}(b_t^*)]
\le
\E[F_t(b_t^*)]
\le
\sqrt{\E[F_t(b_t^*)^2]}
\le
\sqrt{L\E[\bc_t(b_t^*)]}
\le
\sqrt{\frac{LB}{T}} .
\]
Therefore
\[
    \E\parq*{\sum_{t=\tau+1}^{T} r_t^{\bgt}(b_t^*)}
    \le
    (T-\tau)\sqrt{\frac{LB}{T}},
\]
which cancels the negative term in \((\clubsuit)\). Combining the bounds above
and using \(L=O(1)\) \DIFaddbegin \DIFadd{establishes the stated rate on \(\mathcal G\).
}

\DIFadd{It remains to account for the complementary failure event
\(\mathcal G^c\), on which at least one of the HoB-oracle or WLS confidence
events fails. Because all bids, values, and payments are bounded, the
cumulative budget-regret integrand is bounded in
absolute value by \(2T\), the RoS-regret integrand by \(T\), and the
cumulative RoS slack by \((1+\kappa)T\). Therefore, the contribution of
\(\mathcal G^c\) to the unconditional expectations is at most
}\begin{equation}\DIFadd{\label{eq:failure_event_conversion}
    O\!\left((1+\kappa)T\Pr(\mathcal G^c)\right)
    =
    O\!\left((1+\kappa)T^{-2}\right),
}\end{equation}
\DIFadd{which is absorbed into the displayed bounds. This }\DIFaddend completes the proof.
\end{proof}

\subsection{Regret and Violation under RoS Constraint}
\label{app:ros_proofs}

In this section, we first decompose the regret and the violation terms using two generic quantities. Then we substitute the explicit upper bounds for those quantities to obtain the desired results in Section \ref{sec:reg_vio_ros}.

\begin{lemma}[RoS reduction to learning and slack errors]
\label{lem:ros_reduction}
Suppose Assumptions~\ref{ass:unconfound}--\ref{ass:bound_density} and \ref{assump:slater}, the HoB oracle event in
Abstraction~\ref{abs:oracle_hob}, and the WLS event in Lemma~\ref{lem:wls}
hold. Define the Lagrangian error and the slack error respectively as
\[
\calP_T \coloneqq \E \DIFaddbegin \parq*{\sum_{t=1}^T (1+\lambda_t)\parr*{\max_{z_t^*\in\calC_t}\ell_t^0(z_t^*;\rho_t) - \ell_t^0(z_t;\rho_t)}}\DIFaddend ,\qquad \calG_T \coloneqq \E\parq*{\sum_{t=1}^T \bg_t - g_t(b_t)}.
\]
Recall the definition $\lambda_*\in\argmin_{\lambda\ge 0}\overline{D}(\lambda)$. When using $\Lambda \ge 2\lambda_*$ and $\eta = \Theta(T^{-\frac{1}{2}})$, we have
\[
\begin{aligned}
\Reg_T^{\ros}
&\le \calP_T+O(\Lambda\sqrt T),\\
\Viol_T^{\ros}
&\le \frac{2}{\Lambda}
\left(\calP_T+\Lambda\calG_T+O(\Lambda\sqrt T)\right).
\end{aligned}
\]

\end{lemma}
\begin{proof}
By \DIFaddbegin \DIFadd{the definition of \(\rho_t\), for every \(z\in\calC_t\),
}\[
    \DIFadd{(1+\lambda_t)\ell_t^0(z;\rho_t)
    =r_t(z)+\lambda_tg_t(z).
}\]
\DIFadd{By }\DIFaddend weak duality, we have
 \DIFaddbegin \[
\DIFadd{T\cdot\OPT^\ros
\le
\E\parq*{\sum_{t=1}^T(1+\lambda_t)
\max_{z_t^*\in\calC_t}\ell_t^0(z_t^*;\rho_t)}.
}\]
\DIFadd{Therefore,
}\DIFaddend \begin{align*}
\Reg_T^\ros &= T\cdot \OPT^\ros - \E\parq*{\sum_{t=1}^T r_t(b_t)}\\
&\le \E \DIFaddbegin \parq*{\sum_{t=1}^T\parr*{
(1+\lambda_t)\max_{z_t^*\in\calC_t}\ell_t^0(z_t^*;\rho_t)
-r_t(b_t)}}\DIFaddend \\
&= \DIFaddbegin \DIFadd{\calP_T+}\DIFaddend \E \DIFaddbegin \parq*{\sum_{t=1}^T\lambda_tg_t(b_t)}
\DIFaddend \le \calP_T + \E\parq*{\sum_{t=1}^T \lambda_t \bg_t}.
\end{align*}
We first note that $|\bg_t| \le \kappa + 1=O(1)$. By
Lemma \ref{lem:proj_mirror_des} with comparator $\lambda=0$ and
$\lambda_t\le \Lambda$, we have
$\sum_{t=1}^T \lambda_t \bg_t = O(\Lambda\sqrt{T})$, which proves the first
claim. To bound the violation, define the expected cumulative slack
$S_T\coloneqq \E\parq*{\sum_{t=1}^Tg_t(b_t)}$. Then
\begin{align*}
\Reg_T^\ros - \Lambda S_T &\le \E \DIFaddbegin \parq*{\sum_{t=1}^T\parr*{
(1+\lambda_t)\max_{z_t^*\in\calC_t}\ell_t^0(z_t^*;\rho_t)
-r_t(b_t)-\Lambda g_t(b_t)}}\DIFaddend \\
&\overset{(a)}{=} \E \DIFaddbegin \parq*{\sum_{t=1}^T (1+\lambda_t)\parr*{\max_{z_t^*\in\calC_t}\ell_t^0(z_t^*;\rho_t) - \ell_t^0(z_t;\rho_t)} + \sum_{t=1}^T (\lambda_t-\Lambda)g_t(b_t)}\DIFaddend \\
&= \calP_T + \E\parq*{\sum_{t=1}^T (\lambda_t-\Lambda)\bg_t + \Lambda\sum_{t=1}^T (\bg_t - g_t(b_t))}\\
&\overset{(b)}{\le} \calP_T + O(\Lambda\sqrt{T}) + \Lambda\calG_T
\end{align*}
where (a) follows from that the expectation of the randomized bid $b_t$ realizes the hull action $z_t$ in Algorithm \ref{alg:dual-aware-lte}, and (b) applies Lemma \ref{lem:proj_mirror_des} with comparator $\lambda=\Lambda$. From the strong duality in Lemma \ref{lem:convexified_strong_duality}, 
\[
\E\parq*{\sum_{t=1}^T r_t(b_t) + \lambda_* g_t(b_t)} \le T\cdot \overline{D}(\lambda_*) = T\cdot\OPT^\ros.
\]
Consequently, 
\[
\Reg_T^\ros \ge -\lambda_*\Viol_T^\ros = -\lambda_*\max\bra{0, -S_T}.
\]
If $S_T \ge 0$, we are done with $\Viol_T^\ros=0$. Otherwise, we have $\Viol_T^\ros=-S_T$ and
\[
(\Lambda - \lambda_*)\Viol_T^\ros \le \Reg_T^\ros - \Lambda S_T \le \calP_T + O(\Lambda\sqrt{T}) + \Lambda\calG_T.
\]
Since $\Lambda\ge 2\lambda_*$, we have $\Lambda - \lambda_* \ge \frac{\Lambda}{2}$ and hence
\[
\Viol_T^\ros \le \frac{2}{\Lambda}\parr*{\calP_T + O(\Lambda\sqrt{T}) + \Lambda\calG_T}.
\]
\end{proof}

\begin{lemma}[Explicit Error Bounds]\label{lem:bound_P_and_G}
Let $\calB\subseteq [0,1]$ be an evenly discretized grid with size $K$. Suppose the events in Abstraction \ref{abs:oracle_hob} and Lemma \ref{lem:wls} hold. We have
\[
\calP_T = O\parr*{\Lambda\parr*{d\sqrt{T}\log T + \Delta_\hob\sqrt{dT\log T} + H_\hob + \frac{T}{K}}}
\]
and
\[
\calG_T=O\parr*{d\sqrt T\log^2T+\Delta_\hob\sqrt T\log T
+H_\hob+(d^{1/2}+\Delta_\hob^{1/2})T^{3/4}+\Delta_\hob\sqrt T}.
\]
\end{lemma}
\begin{proof}
Let
\(\calI=[T]\setminus(\mathcal E_\hob\cup\mathcal E_{\rm wid}
\cup\mathcal E_{\rm val})\).
On rounds in \(\calI\), exploitation is certified. Lemma
\ref{lem:two_lag_width_valid}, \ref{lem:branching_small_est_err}, and
\ref{lem:one_side_cdf} imply
\[
\max_{z_t^*\in\calC_t}\ell_t^0(z_t^*; \DIFaddbegin \DIFadd{\rho}\DIFaddend _t) - \ell_t^0(b_t; \DIFaddbegin \DIFadd{\rho}\DIFaddend _t) \le \frac{1}{K} + 2w_t^{i_t}(b_t) = O \DIFaddbegin \parr*{\frac{1}{K} + L\hF_t(b_t)(1-\hF_t(b_t))R_t + (1+\rho_t)\delta_t(b_t)}\DIFaddend .
\]
Next we sum it over \(t\in\calI\). Since \(\lambda_t\le\Lambda\) and
\(\hF_t(b_t)(1-\hF_t(b_t))=\sigma_t^{-2}\), by the standard
elliptical-potential bound \eqref{eq:matrix_ellip_potential} and Cauchy--Schwarz inequality, we have
\[
\sum_{t\in\calI}(1+\lambda_t)\! \DIFaddbegin \parr*{\max_{z_t^*\in\calC_t}\ell_t^0(z_t^*;\rho_t)-\ell_t^0(b_t;\rho_t)}
\DIFaddend =
O\parr*{
    \Lambda\parr*{d\sqrt T\log T
    +\Delta_{\hob}\sqrt{dT\log T}
    +\frac{T}{K}}} .
\]
It remains to bound the exploration contribution. The width-based value-exploration rounds have
one-step Lagrangian range \(O(\Lambda)\), hence contribute
\[
O\parr*{\Lambda|\mathcal E_{\rm wid}|}
=
O\parr*{\Lambda\parr*{d\sqrt T\log^2T
+\Delta_{\hob}\sqrt T\log T}} .
\]
Flagged value-exploration rounds have one-step Lagrangian loss \(O(\Lambda
R_t)\) by Lemma~\ref{lem:branching_control}. Since these rounds use
\(\sigma_t=1\), Lemma~\ref{lem:bounded_explore} and \eqref{eq:matrix_ellip_potential} give cumulative contribution
\(O(\Lambda(d\sqrt T\log T+\Delta_{\hob}\sqrt{dT\log T}))\). Combining the
three bounds above proves the displayed
bound for \(\calP_T\).

The HoB-exploration rounds have one-step Lagrangian range \(O(\Lambda)\)
and contribute the displayed \(O(\Lambda H_\hob)\) term exactly once.

Now we consider \(\calG_T\). By the slack UCB definition in
\eqref{eq:slack_ucb}, on WLS-eligible rounds
\(\bg_t(b_t)-g_t(b_t)\le2R_t+2(1+\kappa)\delta_t(b_t)\). Outside
\(\mathcal E_{\rm wid}\cup\mathcal E_\hob\), the value-width test is false, so
\(\sum_{t\notin\mathcal E_{\rm wid}\cup\mathcal E_\hob}R_t
\le(d^{\frac{1}{2}}+\Delta_{\hob}^{\frac{1}{2}})T^{3/4}\). On
\(\mathcal E_{\rm wid}\), use the trivial \(O(1)\) per-round bound and
Lemma~\ref{lem:bounded_explore}. Finally,
\(\sum_{t\notin\mathcal E_\hob}\delta_t(b_t)
\le\Delta_{\hob}\sqrt T\) by Cauchy--Schwarz, while
\(\mathcal E_\hob\) contributes \(O(H_\hob)\) directly. Thus
\[
\calG_T
=
O\parr*{
    d\sqrt T\log^2T
    +\Delta_{\hob}\sqrt T\log T
    +H_\hob
    +
    \parr*{d^{\frac{1}{2}}
    +\Delta_{\hob}^{\frac{1}{2}}}T^{\frac{3}{4}}
    +\Delta_{\hob}\sqrt T}.
\]
\end{proof}
\DIFaddbegin 

\DIFadd{Combining Lemmas~\ref{lem:ros_reduction} and \ref{lem:bound_P_and_G}
establishes the RoS bounds on \(\mathcal G\). The passage to the unconditional
regret and violation definitions follows from the failure-event argument in
the proof of Theorem~\ref{thm:budget_regret}, specifically
\eqref{eq:failure_event_conversion}.
}\DIFaddend

\subsection{Additional Technical Lemmas}

We repeatedly use the standard elliptical-potential bound
\begin{equation}\label{eq:matrix_ellip_potential}
    \sum_{\tau<t}\|z_\tau\|_{A_\tau^{-1}}^2=O(d\log T),
    \qquad
    \sum_{\tau<t}\|z_\tau\|_{A_\tau^{-1}}=O(\sqrt{dt\log T}),
\end{equation}
for \(A_s=I+\sum_{\tau<s}z_\tau z_\tau^\top\) and
\(\|z_\tau\|_2\le1\). See \citet{abbasi2011improved,wen2025joint}.

\begin{lemma}\label{lem:approx_CDF_UCB_Lip}
Let $G$ be a CDF on $[0,1]$ with density upper bounded by $L$, and $\widehat{G}$ be another CDF with $\|G-\widehat{G}\|_\infty \le \delta$ for some $\delta \le \frac{1}{400L^2}$. Let $\calB\subseteq [0,1]$ and $b_0 = \inf\bra{b\in\calB: \widehat{G}(b)\ge \frac{1}{2}}$. Suppose $\exists b\in\calB$ with $b\in [b_0- \zeta,b_0)$ for some $\zeta< \frac{1}{50L}$. Denote $b_*(v) = \argmax_{b\in\calB}\widehat{G}(b)\parr*{v-qb}$ for some $0<q$, with ties broken in favor of the largest maximizer. For any $v_1\le v_2$, if $v_2 - v_1\le \frac{q}{5L}$, then
\begin{align*}
\widehat{G}(\widehat{b}_*(v_2)) - \widehat{G}(\widehat{b}_*(v_1)) \le 1 - \frac{1}{5L}. 
\end{align*}
\end{lemma}
\begin{proof}
Write \(b_i=b_*(v_i)\) and \(s=1/(5L)\). By monotonicity of the empirical
maximizers, Lemma~\ref{lem:monotone_ucb_maximizers} gives
\(\widehat G(b_1)\le \widehat G(b_2)\). Suppose, for contradiction, that
\(\widehat G(b_2)-\widehat G(b_1)>1-s\). Then
\(\widehat G(b_1)\le s\) and \(\widehat G(b_2)\ge 1-s\). Let
\(b_0=\inf\{b\in\calB:\widehat G(b)\ge1/2\}\) and choose
\(b'\in\calB\cap[b_0-\zeta,b_0)\). The uniform approximation
\(\|\widehat G-G\|_\infty\le\delta\) and \(L\)-Lipschitzness imply that
\(b_2\) must lie at least $a=\frac{1/2-s-2\delta}{L}-\zeta$
to the right of \(b_0\). Indeed, if \(b_2<b_0+a\), then
\(b_2-b'<a+\zeta=(1/2-s-2\delta)/L\), so
\[
    \widehat G(b_2)-\widehat G(b')
    \le G(b_2)-G(b')+2\delta
    < 1/2-s.
\]
This contradicts \(\widehat G(b_2)\ge1-s\) and \(\widehat G(b')<1/2\).

Optimality of \(b_2\) against \(b_0\), together with
\(\widehat G(b_0)\ge1/2\) and
\(\widehat G(b_2)-\widehat G(b_0)\le1/2\), gives
\(v_2-qb_2\ge q(b_2-b_0)\). Since \(v_2-v_1\le qs\),
\[
    v_1-qb_0>q(2a-s).
\]
On the other hand, optimality of \(b_1\) against \(b_0\), using
\(\widehat G(b_1)\le s\) and \(\widehat G(b_0)\ge1/2\), gives
\[
    v_1-qb_0\le \frac{2s}{1-2s}qb_0 .
\]
Finally, the existence of \(b'\) just below \(b_0\), the approximation error,
and \(G(1)=1\) imply
\[
    b_0\le 1-\frac{1/2-\delta-L\zeta}{L}.
\]
Combining the two bounds on \(v_1-qb_0\) gives
\[
    q(2a-s)
    <
    v_1-qb_0
    \le
    \frac{2s}{1-2s}qb_0,
\]
and hence \(b_0>\frac{1-2s}{2s}(2a-s)\). For \(s=1/(5L)\),
\(\delta\le1/(400L^2)\), and \(\zeta\le1/(50L)\), this lower bound is at
least the upper bound above, a contradiction. Hence the assumed CDF gap is
impossible, proving
\(\widehat G(b_2)-\widehat G(b_1)\le1-1/(5L)\).
\end{proof}

\begin{lemma}[Optimal bid is monotone in value]\label{lem:monotone_ucb_maximizers}
Let $G:[0,1]\rightarrow[0,1]$ be any CDF. Fix any $B\subseteq [0,1]$. Denote $b_*(v) = \argmax_{b\in B}G(b)\parr*{v-qb}$ for some $q>0$, with ties broken in favor of either the largest or smallest maximizer. Then $b_*$ is increasing in $v$.
\end{lemma}
\begin{proof}
Let $v_1,v_0\in[0,1]$ satisfy $v_1\geq v_0$, and write $b_1=b_*(v_1)$ and $b_0=b_*(v_0)$. For the sake of contradiction, suppose $b_0>b_1$. Then
\begin{align*}
G(b_1)\parr*{v_1-qb_1} &= G(b_1)\parr*{v_0-qb_1} + G(b_1)(v_1-v_0)\\
&\overset{(a)}{\leq} G(b_0)\parr*{v_0-qb_0} + G(b_0)(v_1-v_0)= G(b_0)\parr*{v_1-qb_0}.
\end{align*}
If ties are broken in favor of the largest maximizer, then $b_1=b_*(v_1)\ge b_0>b_1$, a contradiction. If ties are broken in favor of the smallest maximizer, inequality (a) becomes strict because $b_*(v_0)=b_0>b_1$ and hence $G(b_1)\parr*{v_0-qb_1} < G(b_0)\parr*{v_0-qb_0}$. This again contradicts the optimality of \(b_1\) at value \(v_1\), because inequality (a) is then strict.
\end{proof}

\begin{lemma}[Convexified Strong Duality]\label{lem:convexified_strong_duality}
Suppose Assumption \ref{assump:slater} holds. Then
\(\OPT^\ros=\inf_{\lambda\ge0}\E_{x_t}
\max_{(z_F,z_b)\in\calC_t}
\{(1+\lambda)\theta_*^\top x_tz_F-\lambda\kappa z_b\}\).
Moreover, for \(v_t\in[0,1]\) and \(\lambda\ge0\),
\(\max_{(z_F,z_b)\in\calC_t}\{(1+\lambda)v_tz_F-\lambda\kappa z_b\}
=\max_{b\in[0,1]}\{(1+\lambda)v_tF(b)-\lambda\kappa bF(b)\}\).
\end{lemma}
\begin{proof}
The benchmark in \eqref{eq:opt_def_contextual} is a convex program over
measurable selections \(z(x_t)\in\calC_t\), with a linear objective and
constraint and a convex compact action set. Assumption~\ref{assump:slater}
gives strong duality, and pointwise separation over contexts proves the first
claim. The second follows because a linear objective has the same supremum over
\(\calC_t\) and its generating allocation--payment points.
\end{proof}

\begin{lemma}[Projected Mirror Descent]\label{lem:proj_mirror_des}
Let \(h(x)=x\log x-x\) and let \(D_h\) be its Bregman distance. If
\(\nu_{t+1}=\argmin_{\nu\in[0,U]}\{\eta\nu s_t+D_h(\nu,\nu_t)\}\),
\(|s_t|\le G\), and \(\eta G\le1\), then every \(\nu\in[0,U]\) satisfies
\(\sum_{t=1}^T(\nu_t-\nu)s_t
\le D_h(\nu,\nu_1)/\eta+\eta U G^2T\).
\end{lemma}

This standard bound follows from the Bregman three-point inequality. See
\citet{duchi2011adaptive} for a derivation.

\end{document}